\def\year{2021}\relax
\documentclass[letterpaper]{article} 
\usepackage{aaai21}  
\usepackage{times}  
\usepackage{helvet} 
\usepackage{courier}  
\usepackage[hyphens]{url}  
\usepackage{graphicx} 
\usepackage{natbib}  
\usepackage{caption} 
\usepackage{algorithm}
\usepackage[noend]{algorithmic}

\usepackage{mkolar_definitions}
\usepackage{amsmath}
\usepackage{amsthm}
\usepackage{amsbsy}
\usepackage{amsfonts}
\usepackage{multirow}

\usepackage{amssymb}
\usepackage{subfigure}
\usepackage{float}
\usepackage{color}
\usepackage{url}
\newtheorem{assumption}{Assumption}

\newcommand{\sys}[1]{\texttt{LCD}}

\title{Learning to Generate Fair Clusters from Demonstrations}
 \author{Sainyam Galhotra,
Sandhya Saisubramanian and
Shlomo Zilberstein\\}
\affiliations {
University of Massachusetts Amherst\\
}
\nocopyright
\begin{document}
	
	\maketitle

\begin{abstract}
	Fair clustering is the process of grouping similar entities together, while satisfying a mathematically well-defined fairness metric as a constraint. Due to the practical challenges in precise model specification, the prescribed fairness constraints are often incomplete and act as proxies to the intended fairness requirement, leading to biased outcomes when the system is deployed. We examine how to identify the intended fairness constraint for a problem based on limited demonstrations from an expert. Each demonstration is a clustering over a subset of the data. 
	We present an algorithm to identify the fairness metric from demonstrations and generate clusters using existing off-the-shelf clustering techniques, and analyze its theoretical properties. To extend our approach to novel fairness metrics for which clustering algorithms do not currently exist, we present a greedy method for clustering. Additionally, we investigate how to generate interpretable solutions using our approach. Empirical evaluation on three real-world datasets demonstrates the effectiveness of our approach in quickly identifying the underlying fairness and interpretability constraints, which are then used to generate fair and interpretable clusters.  
	
\end{abstract}

\section{Introduction}
Graph clustering is increasingly used for decision making in high-impact applications such as infrastructure development~\cite{hospers2009next}, health care~\cite{haraty2015enhanced}, and criminal justice~\cite{aljrees2016criminal}. These domains involve highly consequential decisions and it is important to ensure that the generated solutions are unbiased. Fair clustering is the process by which similar nodes are grouped together, while satisfying a given fairness constraint~\cite{fairlet}. Prior works on fair clustering focus on designing efficient algorithms to satisfy a given fairness metric~\cite{anderson2020distributional,ahmadian2019clustering,fairlet,fairnesssg,kleindessner2019fair}. These approaches \emph{assume} that the specified fairness metric is complete and accurate. With the increased growth in the number of ways to define and measure fairness, a key challenge for system designers is to accurately specify the fairness metric for a problem.

\begin{figure}
	\centering
	\includegraphics[width=3.5in]{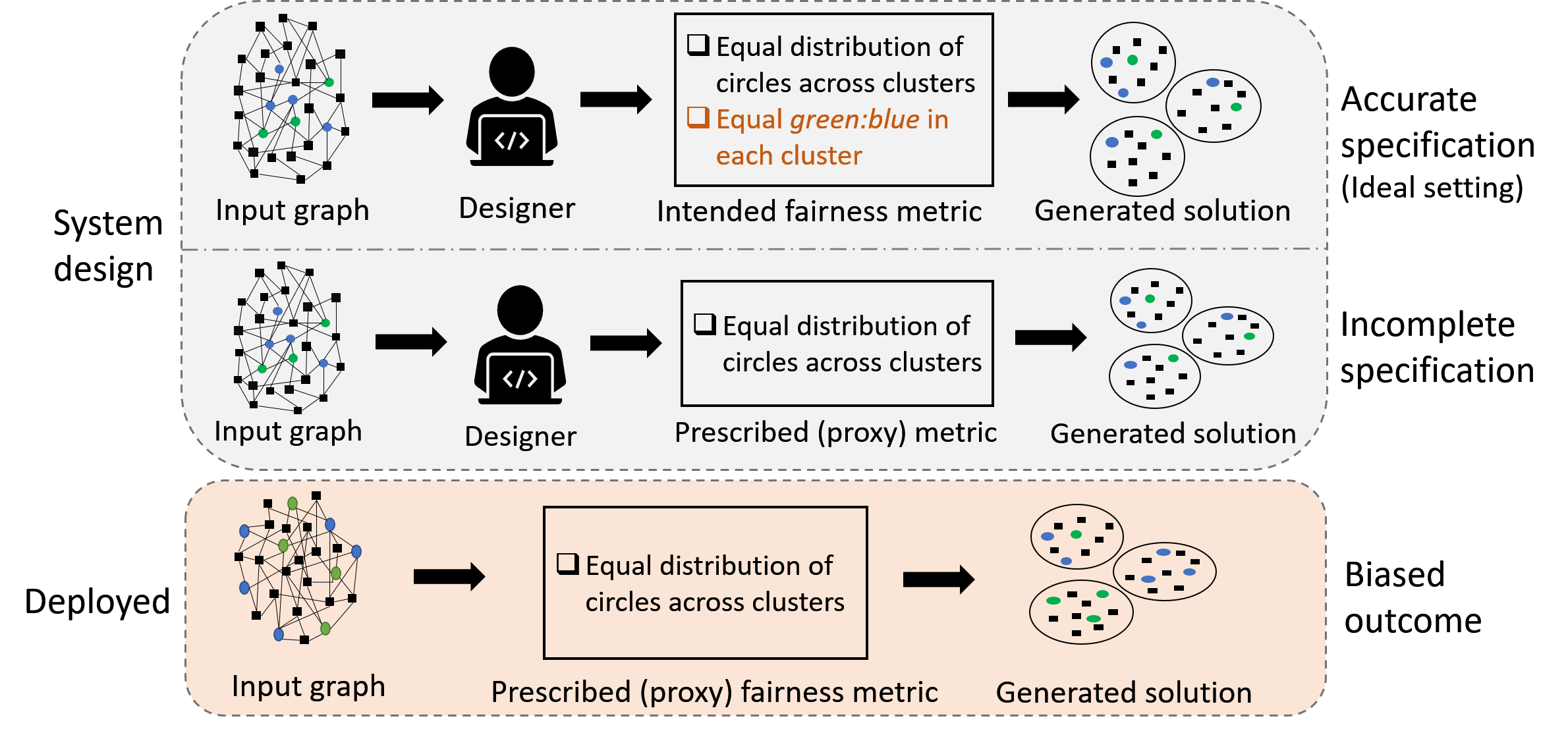}
	\caption{An illustration of incomplete specification of fairness metric leading to biased output---unequal distribution of green and blue nodes in each cluster---when deployed.}
	\label{fig:example}
\end{figure}

Due to the practical challenges in the precise specification of fairness metrics and the complexity of machine learning models, the system's objective function and constraints are often tweaked during the training phase until it produces the desired behavior on a small subset of the data. As a result, the system may be deployed with an incompletely specified fairness metric that acts as a \emph{proxy} to the intended metric. Clustering with incompletely specified fairness metrics may lead to undesirable consequences when deployed. It is challenging to identify the proxies during system design due to the nuances in the fairness definitions and unstated assumptions. Two similar fairness metrics that produce similar solutions during the design and initial testing may generate different solutions that are unfair in different ways to different groups, when deployed. 

For example in Figure~\ref{fig:example}, the designer inadvertently specifies an incomplete fairness metric and assumes the system will behave as intended when deployed. This unintentional incomplete specification is not discovered during the initial testing since the generated results align with that of the intended metric on the training data, such as sample data from California. Consequently, the system may generate biased solution when deployed in Texas, due to demographic shift. Thus, design decisions that seem innocuous during initial testing may have harmful impacts when the system is widely deployed. While the difficulty in selecting a fairness metric for a given problem is acknowledged~\cite{biaswired}, there exists no principled approach to address this meta-problem. \emph{How to correctly identify the fairness metric that the designer intends to optimize for a problem?}


We present an approach that generates fair clusters by learning to identify the intended fairness metric using limited demonstrations from an oracle. 
It is assumed that there exists a true clustering with the intended fairness metrics, which are initially unknown. Each demonstration is a sample from the true clusters, providing information about a subset of the nodes in the dataset. Given a finite number of expert demonstrations, our solution approach first clusters the demonstrations to infer the likelihood of each candidate constraint and then generates clusters using the most likely constraint.  By maintaining a distribution over the candidate metrics and updating it based on the demonstrations, the intended clusters can be recovered since demonstrations are i.i.d. 
The nodes in a demonstration are selected by the expert, abstracted as an oracle. This is in contrast to querying an oracle where the algorithm selects the nodes to query and the oracle responds if they belong to the same cluster or not. When the oracle is a human, demonstrations are easier to collect rather than querying for pairs of nodes, which requires constant oversight.

While inferring the intended fairness metric is critical to minimize the undesirable behavior of the system, the ability of an end user to evaluate a deployed system for fairness and identify when to trust the system hinges on the interpretability of the results. Though clustering results are expected to be inherently interpretable, no clear patterns may be easy to identify when clustering with a large number of features~\cite{saisubramanian2020balancing}. While the existing literature has studied fair clustering and interpretable clustering independently~\cite{fairlet,saisubramanian2020balancing}, to the best of our knowledge, there exists no approach to generate clusters that are both fair and interpretable. We show that our solution approach can generate \emph{fair and interpretable} clusters by inferring both fairness and interpretability constraints, based on limited demonstrations.

Our primary contributions are as follows: (1) formalizing the problem of learning to generate fair clusters from demonstrations; (2) presenting two algorithms to identify the fairness constraints for clustering, generate fair clusters, and analyzing their theoretical guarantees; and (3) empirically demonstrating the effectiveness of our approach in identifying the clustering constraints on three data sets, and using our approach to generate fair and interpretable clusters.

\section{Background and Related Work} \label{sec:prelim}

\vspace{6pt}
\noindent{\textbf{K-center Clustering}}~~ It is one of the most widely studied objectives in the literature~\cite{kcenter2approx}. Let $H\!=\!G(V,d)$ be a graph with $V\!=\!\{v_1,v_2,\ldots,v_n\}$ denoting a set of $n$ nodes, along with a pairwise distance metric $d\!:\!V\times V\!\rightarrow\!\mathbb{R}$. The nodes are described by features values, $F$.
Given a graph instance $H$ and an integer $k$, the goal is to identify $k$ nodes as cluster centers (say $S$, $|S|=k$) and assign each node to the cluster center such that the maximum distance of any node from its cluster center is minimized. The output is a set of clusters $\mathcal{C}\!=\!\{C_1,C_2,\ldots, C_k\}$. The clustering assignment function is defined by $\gamma: V \rightarrow [k]$ and the nodes assigned to a cluster $C_i$ are $\{v \in V \vert \gamma(v) = i \}$.
The objective value is calculated as:
$$o_{kC}(H,\mathcal{C}) = \max_{v\in V} \min_{s\in S} d(u,s). $$
A simple greedy algorithm provides a 2-approximation for the k-center problem and it is NP-hard to find a better approximation factor~\cite{kcenter2approx}. 

\vspace{6pt}
\noindent \textbf{Fairness in Machine Learning}~~ The existing literature on fairness in machine learning can be broadly categorized into two lines of work: defining notions of fairness and designing fair algorithms. 
Various notions of fairness have been studied by researchers in different fields such as AI, Economics, Law, Philosophy, and Public Policy~\cite{bera2019fair,brams1996fair,binns2018fairness,fairlet,fairnesssg,mehrabi2019survey,thomson1983problems,verma2018fairness}. 
The two commonly studied fairness criteria are as follows.

\begin{itemize}
	\item \emph{Group fairness} ensures the outcome distribution is the same for all groups of interest~\cite{fairlet,fairnesssg,verma2018fairness}. This is measured using metrics such as disparate impact~\cite{feldman2015certifying} and statistical parity~\cite{kamishima2012fairness,verma2018fairness} including conditional statistical parity, predictive parity, false positive error rates, and false negative error rates.
	\item \emph{Individual fairness} ensures that any two individuals with the same attributes are not discriminated~\cite{anderson2020distributional,dwork2012fairness,ilvento2019metric}. This is often measured using metrics such as causal discrimination~\cite{dwork2012fairness,verma2018fairness}. 
\end{itemize}
Given a mathematically well-defined fairness criteria, a fair algorithm produces outputs that are aligned with the given fairness definition. Examples include fair clustering~\cite{anderson2020distributional,ahmadian2019clustering,fairlet,kleindessner2019fair}, fair ranking~\cite{celis2018ranking}, and fair voting~\cite{celis2018multiwinner}. Although these works have laid vital ground work to assure fairness in some settings, much of the efforts in designing fair algorithms have focused on the algorithm's performance---efficiency, scalability, and providing theoretical guarantees. There is very little effort, if any, at the \emph{meta-level}: designing algorithms that can identify a suitable fairness metric for a clustering problem, given a set of candidate metrics. 
There has been recent focus on learning a metric~\cite{ilvento2019metric} or a representation that ensures fairness with respect to classification tasks~\cite{hilgard2019learning,gillen2018online}. It is not straightforward to extend these fair classification techniques to fair clustering because they have different objectives. This is further complicated by the lack of ground truth and NP-hardness of clustering. Therefore, it is critical to develop techniques to infer metrics for fair clustering.

\vspace{6pt}
\noindent{\textbf{Fair Clustering}}~~ Fair clustering approaches generate clusters that maximize the clustering objective value, while satisfying the given fairness requirement~\cite{anderson2020distributional,ahmadian2019clustering,bera2019fair,fairlet,kleindessner2019fair}. The commonly considered fairness metrics in clustering are group fairness~\cite{fairlet}, individual fairness~\cite{ilvento2019metric,mahabadifairness20}, and distributional fairness~\cite{anderson2020distributional}. 
These approaches require exact specification of fairness
metrics a priori. They generate fair clusters either by modifying the input graph or use the fairness metrics as constraints and solve it as a linear optimization.


\vspace{6pt}
\noindent{\textbf{Interpretable Clustering}} Interpretable clustering is the process of generating clusters such that it is easy to identify patterns in the data for the end user. A recent approach to generate interpretable clusters maximizes the homogeneity of the nodes in each cluster, with respect to predefined features of interest to the user~\cite{saisubramanian2020balancing}. 
The problem is solved as a multi-objective clustering problem where both interpretability and the k-center objective value are optimized. While both fairness and interpretability are typically investigated independently, the ability to evaluate the system for fairness violations often relies on its interpretability.

\vspace{6pt}
\noindent \textbf{Clustering with an Oracle}~~ Prior works that use additional knowledge from an oracle for clustering typically involve queries of the form `do nodes $u$ and $v$ belong to the same cluster?'~\cite{ashtiani2016clustering,mazumdar2017clustering,mazumdar2017query,galhotra2018robust,firmani2016online,vesdapunt2014crowdsourcing}. Our approach is different from the oracle-based clustering in the following manner. First, our approach involves the oracle selecting the nodes and determining what information is revealed. Second, the oracle provides information potentially about a subset of nodes, instead of pairwise relationships.

\vspace{6pt}
\noindent\textbf{Learning from Demonstration} Learning from demonstration is a type of apprenticeship learning, where the learner learns by observing an expert (typically a human) performing the task~\cite{abbeel2004apprenticeship}. The learner tries to mimic the expert's behavior by observing the demonstrations and generalizing it to unseen situations. Learning from demonstration is a popular approach used to teach robots to complete a task~\cite{abbeel2004apprenticeship} or avoid the negative side effects of their actions~\cite{SKZijcai2020}. 

\vspace{6pt}
\noindent\textbf{Likelihood Estimation}
Maximum likelihood estimation (MLE) 
is a statistical method to estimate the parameters of a probability distribution by maximizing the likelihood function, such that the observed data are most probable under the assumed model~\cite{white1982maximum}. Intuitively, it is a search in the parameter space to identify a set of parameters, for the model, that best fit the observed data. The maximum likelihood estimate is the point in the parameter space that maximizes the likelihood function.

\section{Problem Formulation}\label{sec:problem}
\noindent{\textbf{Problem Statement:~}} Let $G = \langle V, d\rangle$ be the input graph with vertices $V$ and distance metric $d$ and let $o$ denote the clustering objective. Given a finite set of candidate fairness metrics, denoted by $\Omega$, and a finite set of clustering demonstrations, denoted by $\Lambda$, the goal is to identify a fairness metric $\omega_F \in \Omega$ required to be satisfied by the clusters when optimizing objective $o$.

We present \emph{learning to cluster from demonstrations} (LCD), an approach to infer $\omega_F$ using $\Lambda$. LCD is introduced and discussed in the context of fair clustering but it is a generic approach that can be used to infer any clustering constraint. LCD can also handle the case of clustering with multiple fairness metrics by simply considering $\Omega$ to be the power set over possible candidate metrics. 

\vspace{4pt}
\noindent{\textbf{Clustering demonstrations:~}} LCD relies on the availability of clustering demonstrations by an expert. It is relatively easier to gather demonstrations from a human expert than querying for pairs of nodes, which requires constant oversight or availability to answer the queries.

\begin{definition}
	A \textbf{clustering demonstration} $\lambda$ provides the inter-cluster and intra-cluster links for a subset of nodes from the dataset $S\!\subseteq\!V, \vert S \vert \geq\!2,$ by grouping them according to the underlying objective function and constraints, $\lambda\!=\!\{ C_1, \ldots, C_t \}$ with each $C_i$ denoting a cluster and $t \leq k$.	
\end{definition}
To generate a demonstration, the oracle selects a subset of nodes and then clusters it, in accordance with the true clusters. The following assumption ensures that demonstrations are i.i.d and the expert is not acting as an adversary. 
\begin{assumption}
	The nodes in each demonstration are randomly selected and clustered according to the ground-truth fairness constraints.
\end{assumption}

Therefore, a demonstration $\lambda$ is a sample of the underlying clustering, revealing the relationship between a subset of the nodes. However the relationship between the nodes in successive demonstrations is \emph{unknown}, when the nodes are distinct in each demonstration. We illustrate this with an example. Consider seven nodes $\{u_1,\ldots, u_7\}$ whose true but initially unknown clustering is $C^*_1 = \{ u_1, u_2, u_3\}$, $C^*_2 = \{ u_4, u_5\}$, and $C^*_3 = \{ u_6, u_7\}$ . Let $\lambda_1 = \{ (u_1, u_2), (u_4) \}$ and $\lambda_2 = \{ (u_3), (u_5),(u_6) \}$ denote two successive demonstrations. Demonstration $\lambda_1$ shows that $u_1, u_2$ are in the same cluster and $u_4$ is in a separate cluster. Demonstration $\lambda_2$ shows that $u_3$, $u_5$ and $u_6$ are in different clusters. At the end of $\lambda_1$ and $\lambda_2$, it is not clear if $u_1, u_2$ and $u_3$ belong to the same cluster or not. 

\begin{definition}
	\textbf{Globally informative demonstration} provides the true cluster affiliation of a subset of nodes,  $S \subseteq V$, and is denoted by $\lambda_g = \{ \langle u_1,\gamma(u_1)\rangle, \ldots, \langle u_s,\gamma(u_s)\rangle \}$, $ \forall u_i \in S $ with $\gamma(u)$ indicating the cluster affiliation of node $u$. \label{defn:global}
\end{definition}
Globally informative demonstration provides information about the true cluster affiliation (cluster ID) of the nodes, which is used to retrieve the inter-cluster and intra-cluster links between the nodes and form clusters  $\{ C_1, \ldots, C_t \}$ with $t \leq k$. The information provided by a \emph{single} globally informative demonstration is the same as a regular clustering demonstration. However, globally informative demonstrations facilitate cross-referencing the cluster affiliations across demonstrations, overcoming the drawback of general
clustering demonstration. Consider the example with global demonstrations $\lambda_1 = \{ \langle u_1,1\rangle, \langle u_2,1\rangle, \langle u_4,2\rangle \}$ and $\lambda_2 = \{ \langle u_3,1\rangle, \langle u_5,2\rangle, \langle u_6,3\rangle \}$. Then we know that $C^*_1 = \{ u_1, u_2, u_3\}$. This subtle but important distinction accelerates the identification of fairness constraints. 



\begin{table*}
	\begin{center}
		\resizebox{\textwidth}{!}{%
			\begin{tabular}{|c|c|c|c| } 
				\hline
				Symbol & Formula & Parameter& Reference \\ 
				\hline\hline
				$\omega_{GF}$ & Ratio of  each feature value  $\in [\alpha,\beta]$ &$ \alpha,\beta$&\cite{bera2019fair,fairlet}\\ \hline
				$\omega_{EQ}$ & Relative distribution of a specific feature value &$ \beta$ &\cite{ding2020faster,galhotra2019lexicographically} \\ \hline
				$\omega_{IC}$ & Homogeneity of clusters &$ \beta$ &\cite{saisubramanian2020balancing}\\ 
				\hline
			\end{tabular}
		}
	\end{center}
	\caption{Candidate fairness and interpretable constraints ($\Omega$). \label{tab:constraint}}
	\vspace{-4pt}
\end{table*}

\subsection{Fairness and Interpretability Constraints\label{sec:constraints}}
In the rest of the paper, we focus on inferring the following constraints, with  constraint thresholds defined below.

\vspace{4pt}
\noindent \textbf{Disparate impact or group fairness ($\bm{\omega_{GF}}$).}~~ This commonly studied fairness metric requires the fraction of nodes belonging to all groups, characterized by the sensitive feature, to have a fair representation in each cluster. Suppose the sensitive feature takes two values---Red or Blue, with each node assigned one of the two colors. This constraint requires the fraction of red and blue nodes in a cluster to be within $[\alpha,\beta]$ where $\alpha,\beta\!\in\![0,1]$ are called \emph{constraint thresholds}~\cite{bera2019fair,fairlet}. 

\vspace{4pt}
\noindent \textbf{Equal representation ($\bm{\omega_{EQ}}$).}~~ This fairness constraint enforces equal distribution of nodes with a specific feature value, across clusters. An example is requiring all clusters to have equal number of nodes with the feature value `Red'. This clustering constraint has been particularly  useful in team formation settings, where the resources are fixed and certain colored nodes need to be distributed equally among teams (clusters). More formally, let $\alpha_i$ denote the number of nodes with feature value $\alpha$ in cluster $C_i$. Constraint $\omega_{EQ}$ requires $\alpha_i=\alpha_j$. Restricting all nodes of feature value $\alpha$ to be distributed equally may be very strict for some applications. A generalization of this constraint requires the distribution ratio to be greater than a pre-defined threshold $\beta$,  $\frac{\alpha_i}{\alpha_j} > \beta$, for every pair of clusters~\cite{ding2020faster,galhotra2019lexicographically}.
This ratio captures the relative distribution of $\alpha$-valued nodes across the clusters.

\vspace{4pt}
\noindent \textbf{Interpretability ($\bm{\omega_{IC}}$).}~~ This constraint considers a specific feature of interest (say `Color') and requires that all clusters are homogenized according to the considered feature. The homogeneity of a cluster with respect to a feature $f$ is characterized by the fraction of nodes of a cluster that have same feature value for the input feature. For example, consider a cluster with $7$ blue nodes, $2$ red nodes and $1$ green colored node. Then the homogeneity of the cluster with respect to the feature `color' and feature value `blue' is $0.7$. Generating interpretable clusters requires satisfying a homogeneity threshold $\beta$--- each cluster is required to have at least $\beta$ fraction of nodes with respect to $f$~\cite{saisubramanian2020balancing}.\\

These constraints, described by a feature $f$ and a threshold $\beta$, are summarized in Table~\ref{tab:constraint}. Given the set of candidate constraints $\Omega$ and demonstrations $\Lambda$, LCD aims to identify the constraint, along with its feature and corresponding threshold, that has the maximum likelihood. 

\section{Solution Approach} \label{sec:alg}
We begin by describing a naive approach to infer the constraint thresholds and discuss its limitations. We then propose an algorithm that infers the constraint threshold and generates clusters using existing clustering algorithms. To extend our approach to handle fairness metrics that are not currently supported by the existing algorithms, we present a greedy clustering approach. 

\subsection{Naive algorithm}
A naive approach to infer the clustering constraint from a given set of demonstrations $\Lambda$ is to exhaustively generate all possible clusterings for each type of constraint, its corresponding feature, and  threshold. Among these clusterings, the most likely set of clusters correspond to the one having maximum conformance with the demonstrations $\Lambda$. This approach is highly effective in identifying the desired set of clusters but does not scale, given that the fairness constraint threshold can take infinite values. For example, the disparate impact constraint $\omega_{GF}$ take two parameters $\alpha,\beta $ as input, which can take any value in the range $[0,1]$. To efficiently infer the constraint, we build on the following observations.
\begin{itemize}
	\item $k$-center clustering (and centroid-based clustering in general) aims to minimize the maximum distance of any node from the cluster center. Therefore, it is very unlikely that a particular node is assigned to the farthest center.
	\item Our problem can be modeled as a likelihood estimation problem, where the most likely constraint is expected to correspond to the ground truth constraint.
\end{itemize} 

Given a cluster $C$, we can estimate the most likely threshold of $C$ with respect to a constraint, by following the procedure discussed in the previous section. For example, if a cluster has $3$ red nodes and $5$ blue nodes, we can infer that the fraction of nodes of each color is at least $\min (3/8,5/8)$. Using this constraint threshold estimation, 
a simple approach is to estimate the likelihood of different clustering constraints by considering each demonstration as an independent set of clusters and calculate  threshold with respect to each constraint over these clusters. A major drawback of this approach is that a single clustering demonstration generally does not contain representation from all $k$ clusters and feature values for the considered feature. This may mislead the likelihood estimation when a demonstration considered in isolation. 

\begin{example}
Consider an optimal clustering for $\bm{\omega_{GF}}$, denoted by $C_1 = \{r_1,r_2,b_1,b_2\}$ and $C_2 = \{r_3,b_3\}$, where $r_1,r_2,r_3$ are the red nodes and $b_1,b_2,b_3$ are blue colored nodes. Suppose one of the demonstration  is $\lambda=\{(r_1,r_2),(r_3)\}$. Based on this demonstration, the inferred constraint is $\bm{\omega_{IC}}$ with $\beta=1$, which incorrectly indicates that all the nodes in a cluster have the same color.
\end{example}

\subsection{Proposed Algorithm}
We present Algorithm~\ref{alg:likelihood} that clusters the given demonstrations and processes these clusters to infer the most likely constraint and its parameter values (feature and threshold). Figure~\ref{fig:pipeline} presents the high level architecture of our proposed technique. Given a collection of demonstrations generated by an expert, our algorithm greedily merges them to generate $k$ clusters. These clusters are then used to calculate the likelihood of each fairness constraint and infers the clustering with maximum likelihood.

\begin{figure}
	\includegraphics[width=\columnwidth]{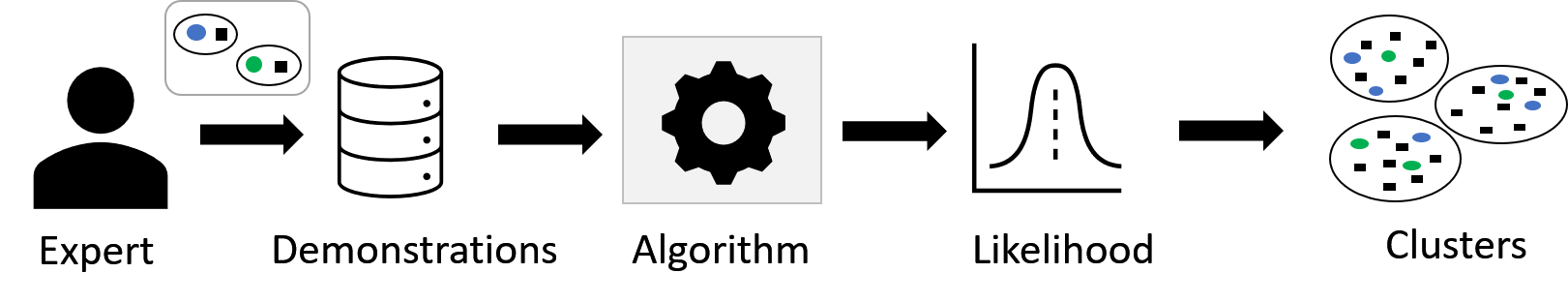}
	\caption{Overview of solution approach.}
	\label{fig:pipeline}
\end{figure}

\begin{algorithm}[t]
	\caption{Maximum Likelihood Constraint \label{alg:likelihood}}
	\begin{algorithmic}[1]
		\REQUIRE Demos $\Lambda$, Nodes $V$, Features of interest $F$
		\ENSURE Clusters $\mathcal{C}$
		\FOR{$v\in \Lambda$}
		\STATE $C\leftarrow C\cup \{v\}$
		\ENDFOR
		\STATE $\mathcal{C}\leftarrow \texttt{ConstructClusters}(\Lambda)$
		\WHILE{$|\mathcal{C}|>k$}
		\STATE $\mathcal{C}\leftarrow \texttt{MergeClosest}(\mathcal{C})$
		\ENDWHILE
		\STATE $T(\omega,f)\leftarrow 0, \forall \omega \in \Omega ,f\in F$
		\FOR{$\omega\in \Omega ,f\in F$}
		\STATE $T(\omega,f)\leftarrow \texttt{CalculateThreshold}(C,\omega,f)$
		\ENDFOR
		\FOR{$(\omega,f)\in T$}
		\STATE $C_{\omega,f}\leftarrow \textsc{Cluster}(\omega,f,V)$
		\STATE $ \mathcal{L}_{\omega,f}\leftarrow \texttt{Likelihood}(C_{\omega,f},\Lambda)$
		\ENDFOR
		\STATE $(\omega,f) \leftarrow \arg \max (\mathcal{L}_{\omega,f})$
	\end{algorithmic}
\end{algorithm}

Algorithm~\ref{alg:likelihood} proceeds in two phases. In the \emph{first phase} (Lines 1-5), the algorithm forms $k$ clusters of the demonstrations $\Lambda$.  This phase initializes a clustering $\mathcal{C}$ over the set of nodes in demonstrations $\Lambda$ ($\texttt{ConstructClusters}(\Lambda)$) which correspond to the different clusters identified by the expert. 
Note that the set $\mathcal{C}$ may contain more than $k$ clusters. In that case, we greedily merge the closest pair of clusters until $k$ clusters have been identified. The distance between any pair of clusters $C_i, C_j\in \mathcal{C}$ is measured as the maximum distance between any pair of nodes in $C_i$ and $C_j$:
\begin{align*}
d(C_1,C_2) = \max_{u\in C_1, v\in C_2} d(u,v).
\end{align*}

In the \emph{second phase} (Lines 6-12), the identified clusters $\mathcal{C}$ are processed to calculate the most likely threshold with respect to each feature and constraint (denoted by $T$). The identified threshold is used to generate a set of $k$ clusters on the original dataset $V$ for each $\langle$constraint, feature$\rangle$ pair. At the end of this step, there are $|F| \times |\Omega|$ clusterings, with one of them corresponding to the intended set of clusters. 

To identify the set of clusters with maximum likelihood ($\mathcal{L}$), we calculate the accuracy of each clustering with respect to the input demonstrations and return the set of clusters that have the highest accuracy. The accuracy of a set of clusters $\mathcal{C}$ is calculated by labeling each pair of nodes as intra-cluster or inter-cluster, and then measuring the fraction of pairs that have same labels according to $\mathcal{C}$ and $\Lambda$. The accuracy estimate of the clusters $\mathcal{C}$ captures the likelihood of a particular constraint.

\paragraph{Complexity.} The first phase of Algorithm~\ref{alg:likelihood} is initialized with $O(|\Lambda|)$ demonstrations and iteratively reduced to $k$ clusters. In each iteration, it calculates the distance between pairs of clusters, resulting in $O(|\Lambda^2|)$ run time. The second phase considers all combinations of constraint and features, thereby performing clustering $|F|\times |\Omega|$ times where $F$ denotes the set of features for each node. Therefore, the run time complexity of Algorithm~\ref{alg:likelihood} to calculate clusters over the demonstrations is $O(\log^3 n)$ and it takes $O(n |F| |\Omega|)$ to construct clusters and calculate likelihood.

Algorithm~\ref{alg:likelihood} identifies the optimal set of clusters and the maximum likelihood constraints for a given set of demonstrations, assuming that a clustering technique \emph{exists} for an input constraint.  We now present a  greedy algorithm that does not rely on the clustering technique and greedily generates the set of clusters with maximum likelihood.

\begin{algorithm}[t]
	\caption{Greedy Algorithm for Novel Metrics\label{alg:greedy2}}
	\begin{algorithmic}[1]
		\REQUIRE Demos $\Lambda$, Nodes $V$, Features of interest $F$
		\ENSURE Clusters $\mathcal{C}$
		\FOR{$v\in V$}
		\STATE $C\leftarrow C\cup \{v\}$
		\ENDFOR
		\STATE $\mathcal{C}\leftarrow \texttt{ConstructClusters}(\Lambda)$
		\WHILE{$|\mathcal{C}|>k$}
		\STATE $\mathcal{C}\leftarrow \texttt{MergeClosest}(\mathcal{C})$
		\ENDWHILE
		\STATE $T\leftarrow$ Calculate constraint threshold of each constraint
		\FOR{$(\omega,f)\in L$}
		\STATE $\mathcal{L}(\omega,f)\leftarrow $ Calculate likelihood of each constraint
		\STATE Perform greedy adjustment to satisfy each constraint
		\STATE $ \mathcal{L}_{\omega,f}\leftarrow \texttt{Likelihood}(C_{\omega,f},\Lambda)$
		\ENDFOR
		\STATE Return the clustering corresponding $\arg \max (\mathcal{L}_{\omega,f})$
	\end{algorithmic}
\end{algorithm}

\subsection{Greedy Algorithm for Novel Metrics}
To handle the fairness objectives for which fair clustering algorithms do not currently exist, we present a greedy algorithm that generates $k$ clusters without assuming any knowledge about the clustering algorithm for the input constraints.

Our approach is outlined in Algorithm~\ref{alg:greedy2}. Given a collection of demonstrations $\Lambda$ and vertices $V$, the algorithm proceeds in two phases. The first phase of Algorithm~\ref{alg:greedy2} (Lines 1-5) is similar to that of Algorithm~\ref{alg:likelihood}, where all nodes are initialized as singleton clusters and all nodes that are grouped together in $\Lambda$ are merged. The closest pair of clusters in $\mathcal{C}$ are sequentially merged until $k$ clusters have been identified. Let $\mathcal{C}$ denote the final set of $k$ clusters. 

The second phase (Lines 6-12) begins with estimating the constraint threshold ($T$), as in Algorithm~\ref{alg:likelihood}. The estimated threshold is used to greedily post-process the clusters according to each constraint. This greedy processing transfers the nodes from one cluster to another, following the constraint requirements and is similar to local search techniques that move nodes between clusters to satisfy a constraint. At the end of this phase, there are $|F|\times |\Omega|$ different sets of clusters, with each optimizing a different fairness constraint. The clustering that has the highest likelihood with the input demonstrations is returned as the final set of clusters. The likelihood is estimated in terms of the accuracy of pairwise intra-cluster and inter-cluster labels.

\section{Theoretical Analysis\label{sec:theory}}
In this section, we analyze the effectiveness of Algorithm~\ref{alg:likelihood} to identify the constraints even when the oracle presents  $\Theta(\log n)$ demonstrations, where $n = |V|$. We first show that the estimated constraint is accurate with a high probability under the assumption that the oracle chooses nodes uniformly at random. We then extend the analysis to settings where the presented demonstrations are biased towards specific clusters. This analysis assumes that each demonstration $\lambda\in \Lambda$ has constant size\footnote{ Our proofs extend to the setting where demonstration size is $\Omega(1)$ too.}.

Let $\Tilde{V}$ denote the set of nodes that have been clustered in atleast one of the demonstrations. 
Lemma~\ref{lem:sample} shows that the sample $\Tilde{V}$ contains $\Theta(\log n)$ from a cluster $C^*$ whenever $|C^*|\geq \frac{n}{k}$.

\begin{lemma}
	Consider a random sample $\Tilde{V}\subseteq V$ such that  $|\Tilde{V}|\geq 10~k \log n$ and each node  in $\Tilde{V}$ is chosen uniformly at random,  then $|\Tilde{V} \cap C^*| \geq 5\log n$, $\forall C^* \geq \frac{n}{k}$.\label{lem:sample}
\end{lemma}
\begin{proof}
	Let $X_v$ be a binary indicator variable such that $X_v=1$ if $v\in \Tilde{V}$ and $0$ otherwise.
	Since, each record $v$ is chosen uniformly at random, $Pr[v\text{ is chosen}] = \frac{|\Tilde{V}|}{n}$.
	Therefore, $$E\left[|\Tilde{V} \cap C^*| \right] \geq  \frac{|\Tilde{V}|}{n} |C^*| = 10 \log n.$$
	Using Chernoff bound, $|\Tilde{V} \cap C^*| \geq 5\log n$ with a probability of $1-\frac{1}{n^2}$.
\end{proof}

Consider a set of ground truth clusters, $\mathcal{C}^*=\{C_1^*,\ldots, C_k^*\}$, such that $\forall |C_i^*|$ satisfy one of the clustering constraint $\omega \in \Omega$. This means that $\exists i$ such that $|C_i^*|\geq \frac{n}{k}$. For the next part of the proof, we will consider this $C_i^*$ to analyze the quality of estimated constraint threshold. 

\begin{lemma}
	Suppose the optimal cluster $C_i^*$ satisfies the constraint, $\omega_{GF}$ with parameters $[\alpha,\beta]$ and $|\Tilde{V}\cap C_i^*|=\Theta(\log n)$, then the estimated threshold on processing $|\Tilde{V}\cap C_i^*|$ is $[\alpha(1-\epsilon),\beta(1+\epsilon)]$ with a high probability.\label{lem:gf}
\end{lemma}
\begin{proof}
	Suppose the optimal fairness constraint $\omega_{GF}$ considers a feature $f$ with parameters $[\alpha,\beta]$. Let $A=\{a_1,\ldots, a_t\}$ denote the domain of values for the feature $f$. According to the fairness constraint, the subset of $C_i^*$ that has feature value $a_j, \forall j$ is within a fraction $[\alpha,\beta]$.  Suppose the fraction of nodes with feature  value $a_i$ be $\alpha_i$.
	
	We claim that the fraction of nodes with feature $\alpha_i$ in the sample $\Tilde{V}\cap C^*$ is within $[\alpha_i(1-\epsilon),\alpha_i(1+\epsilon)]$ with a high probability, where $\epsilon$ is a small constant.
	Let $X_v$ denote a binary random variable such that $X_v$ is one if $v$ is present in the sample $\Tilde{V}$ and $0$ otherwise.
	The expected number of nodes that have feature $\alpha_i$ and belong to the set $\Tilde{V}\cap C_i^*$ is $\frac{\alpha_i |C_i^*| |\Tilde{V}|}{n} = \Theta(\log n)$.
	Following the proof of Lemma~\ref{lem:sample} and using Chernoff bound, we get that the number of nodes with value $a_j$ is within a factor of  $[(1-\epsilon/2),(1+\epsilon/2)]$ of the expected value with a high probability.
	Additionally, the expected number of nodes that belong to the sample $|C_i^*\cap \Tilde{V}|= \frac{|C_i^*| |\Tilde{V}|}{n} $ and the number of nodes is within a factor of $[(1-\epsilon/2),(1+\epsilon/2)]$ with a high probability.
	
	Therefore, the ratio of node that have feature value $a_i$ and belong to the sample $\Tilde{V}\cap C_i^*$ is always within a factor of $\left[\frac{1-\epsilon/2}{1+\epsilon/2},\frac{1+\epsilon/2}{1-\epsilon/2}\right] \sim [1-\epsilon,1+\epsilon]$ for small values of $\epsilon$. Taking a union bound over all feature values, we guarantee that the estimated parameter is within a factor of $[1-\epsilon,1+\epsilon]$ with a high probability.
\end{proof}

\begin{lemma}
	Suppose the optimal cluster $C_i^*$ satisfies the constraint, $\omega_{IC}$ with parameter $\beta$ (some constant) and $|\Tilde{V}\cap C_i^*|=\Theta(\log n)$, then the estimated threshold on processing $|\Tilde{V}\cap C_i^*|$ is $[\beta (1-\epsilon),\beta(1+\epsilon)]$ with a high probability.\label{lem:ic}
\end{lemma}
\begin{proof}
	Suppose the optimal cluster $C_i^*$ satisfies $\omega_{IC}$  with parameter $\beta$ with respect to a feature value $\alpha$. Therefore, $\beta$ fraction of the nodes in $C_i^*$ have the feature value $\alpha$. To analyze the fraction of nodes of feature value $\alpha$, we define binary random variable $X_v$ for each $v$ such that $X_v=1$ if $v\in \Tilde{V}$ and $0$ otherwise. The expected number of nodes with feature value $\alpha$ in the sample $\Tilde{V}\cap C_i^*$ is $\frac{\beta |C_i^*||\Tilde{V}|}{n}$. Following the analysis of Lemma~\ref{lem:gf}, we get that the fraction of nodes of color $\alpha$ is within a factor of $[\beta (1-\epsilon),\beta(1+\epsilon)]$ with a probability of $1-\frac{1}{n}$.
\end{proof}

\begin{lemma}
Suppose the optimal cluster $C_i^*$ satisfies the constraint, $\omega_{EQ}$ with parameter $\beta$ and $|\Tilde{V}\cap C_i^*|=\Theta(\log n)$,, then the estimated threshold on processing $|\Tilde{V}\cap C_i^*|$ is $[\beta (1-\epsilon),\beta(1+\epsilon)]$ with a high probability.\label{lem:eq}
\end{lemma}
\begin{proof}
This analysis is similar to that of Lemma~\ref{lem:ic}.
\end{proof}

Lemmas~\ref{lem:gf},~\ref{lem:ic} and~\ref{lem:eq}  show that the estimated parameter  from a cluster $C_i^*$ with respect to the considered fairness constraints is within a factor of $[(1-\epsilon),(1+\epsilon)]$ of the true constraint threshold with a high probability. Using these results, we prove the following theorem.

\begin{theorem}
Given a collection of nodes $V$ and randomly chosen globally informative demonstrations $\Lambda = \Theta(\log n)$ such that each demonstration reveals the true cluster affiliation of a constant number of records, then the optimal cluster constraint  is identified within a multiplicative factor of $[(1-\epsilon),(1+\epsilon)]$  with a high probability.\label{thm:demo2}
\end{theorem}
\begin{proof}
Let $\Lambda$ denote a collection of globally informative demonstrations such that $|\Lambda|\!=\!\Theta(\log n)$ and let $\Tilde{V}\!= \cup_{\lambda_g\in \Lambda} \lambda_g$.  Using Lemma~\ref{lem:sample}, we know that $\Tilde{V} \cap C_i^* =\!\Theta(\log n)$  for all $C_i^*$ containing $\Theta(n)$ nodes and therefore, using Lemmas~\ref{lem:gf},~\ref{lem:ic} and~\ref{lem:eq} we are guaranteed to estimate the correct threshold for the cluster $C_i^*$. Hence, Algorithm~\ref{alg:likelihood} correctly estimates the constraint with maximum likelihood with $\Theta(\log n)$ globally informative demonstrations.
\end{proof}

\begin{remark}
In this section we do not optimize for the constants in $\Theta$ notation because Algorithm~\ref{alg:likelihood} empirically converges in less than $2\log n$ demonstrations.
\end{remark}

We extend the proof of Theorem~\ref{thm:demo2} to the setting where the demonstrations are not globally informative but the ground truth clusters satisfy an interesting property, similar to the $\gamma$-margin property studied in prior work~\cite{ashtiani2016clustering}. We first define the margin property. Let $\Tilde{V}$ denote a subset of nodes and $\mathcal{C}^*$ denote the set of clusters corresponding optimal constraint. The set $\Tilde{V}$ is considered to satisfy margin property if $d(u,x) > d(u,v)$ where $u,v\in C_i^*\cap \Tilde{V}$ and  $x\in \Tilde{V}\setminus C_i^*$.

\begin{theorem}
Given a collection of nodes $V$ and randomly chosen demonstrations $\Lambda = \Theta(\log n)$ such that each demonstration reveals the clustering over a subset of nodes, then the optimal cluster constraint  is identified within a multiplicative factor of  $[(1-\epsilon),(1+\epsilon)]$  with a high probability if the sampled nodes  $\cup_{\lambda\in \Lambda} \lambda $ satisfy the margin property.\label{thm:likelihood}
\end{theorem}
\begin{proof}
Let $\Lambda$ denote a collection of demonstrations such that $|\Lambda|=\Theta(\log n)$ and let $\Tilde{V} = \cup_{\lambda\in \Lambda} \lambda$.  Using Lemma~\ref{lem:sample}, we know that $\Tilde{V} \cap C_i^* = \Theta(\log n)$  for all $C_i^*$ containing $\Theta(n)$ nodes. This guarantees that we have $\Theta(\log n)$ nodes sampled from $\mathcal{C}_i^*$ but we may not have merged all these nodes to form a single cluster. In order to show that the nodes present in merged cluster (after Line 5 of Algorithm~\ref{alg:likelihood}) belong to the same cluster, we use the margin property. The margin property assumes that all nodes that belong to same cluster are closer to each other as compared to nodes of other cluster. Therefore, \texttt{MergeClosest} always merges a pair of clusters that belong to same optimal cluster $C_i^*$, thereby guaranteeing its correctness. Since $C_i^*$ has been constructed correctly, the proof is same as Theorem~\ref{thm:demo2}.
\end{proof}

\begin{theorem}
	Given a collection of nodes $V$ and randomly chosen demonstrations $\Lambda = \Theta(\log n)$ such that each demonstration reveals the clustering over a subset of nodes, then Algorithm~\ref{alg:greedy2} recovers  ground truth clusters with a high probability if the nodes  $V$ satisfy the margin property.\label{thm:greedy2}
\end{theorem}
\begin{proof}
	This analysis is similar to that of Theorem~\ref{thm:likelihood}.
\end{proof}

\noindent \textbf{Discussion.}
The analysis of Theorem~\ref{thm:likelihood} assumed margin property over the sampled nodes. In most real world datasets, clusters are often well separated, thereby automatically implying the margin property. Additionally, even if the margin property does not hold on overall clusters, expert can choose samples for the demonstration such that the samples of different clusters are present sufficiently away. The proof of Theorem~\ref{thm:likelihood} can be extended to settings where a constant fraction of sampled nodes do not obey the margin property.

Another important assumption that is crucial in the analysis presented above is the randomness of sampled nodes. Theorem~\ref{thm:demo2} and \ref{thm:likelihood} assume that every node is chosen uniformly at random. Note that these assumptions can be relaxed and our proofs extend to settings when the \emph{samples are biased towards a specific cluster}. For example, the number of samples a specific cluster (say $C_i^*$) is much higher than $\Theta(\log n)$ but the samples from other clusters are much fewer. In this case, Algorithm~\ref{alg:likelihood}  will correctly estimate the threshold from $C_i^*$ with fewer demonstrations but it may require more number of demonstrations to achieve accurate estimate from other clusters. 

\begin{figure*}
\subfigure[\texttt{Bank}, $\omega_{GF}$]{\includegraphics[width=0.24\textwidth]{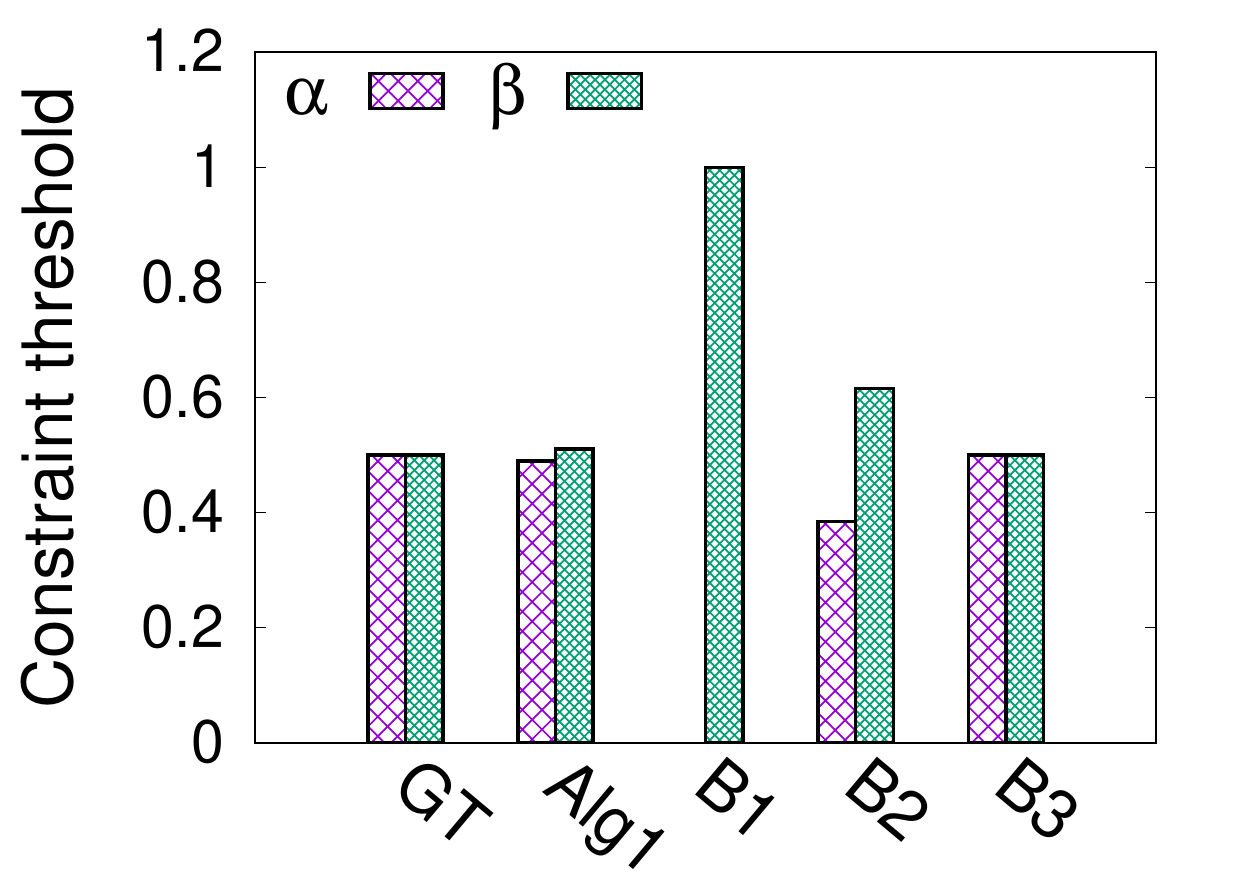}\label{constraint-bank}}
\subfigure[\texttt{Adult}, $\omega_{EQ}$]{\includegraphics[width=0.24\textwidth]{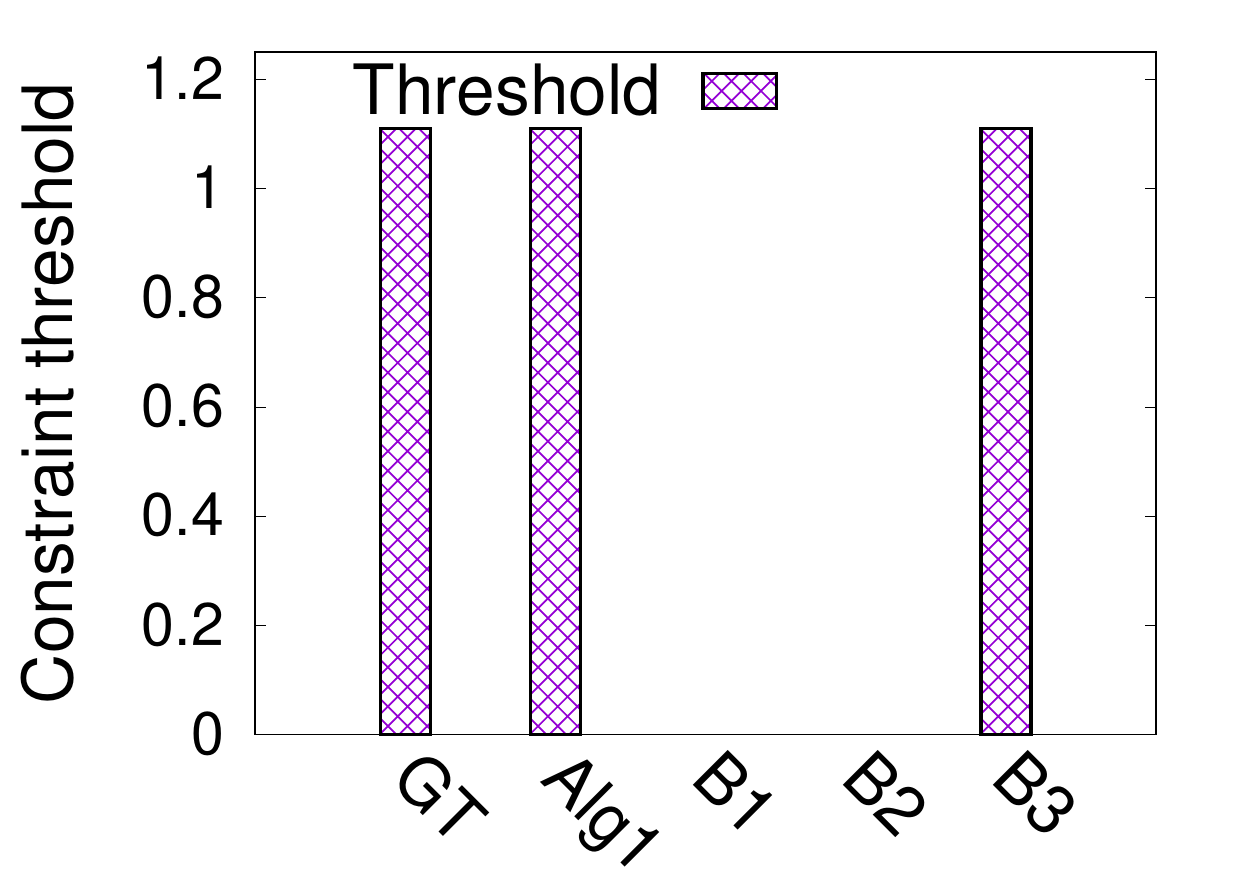}\label{constraint-tf}}
\subfigure[\texttt{Crime}, $\omega_{IC}$]{\includegraphics[width=0.24\textwidth]{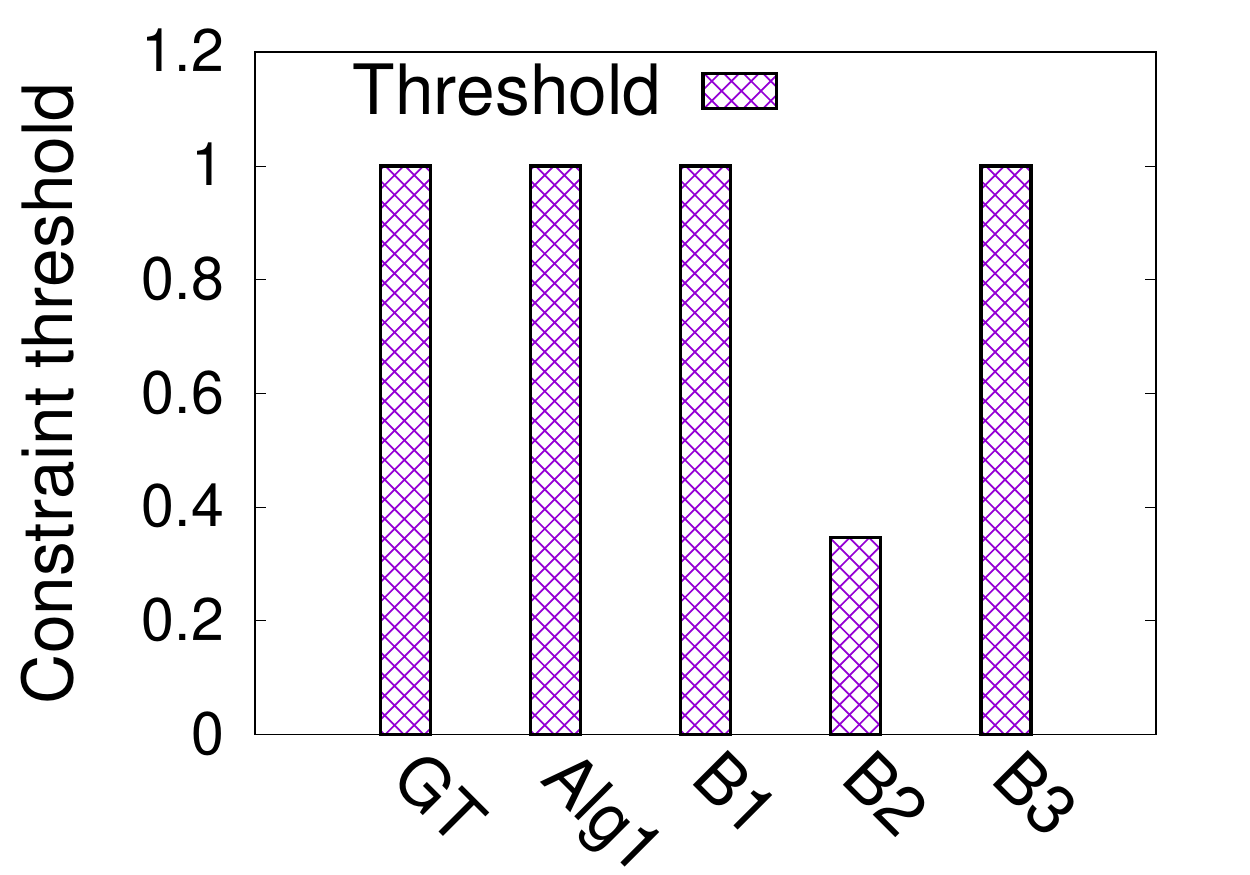}\label{constraint-crime1}}
\subfigure[\texttt{Adult}, $\omega_{IC}$]{\includegraphics[width=0.24\textwidth]{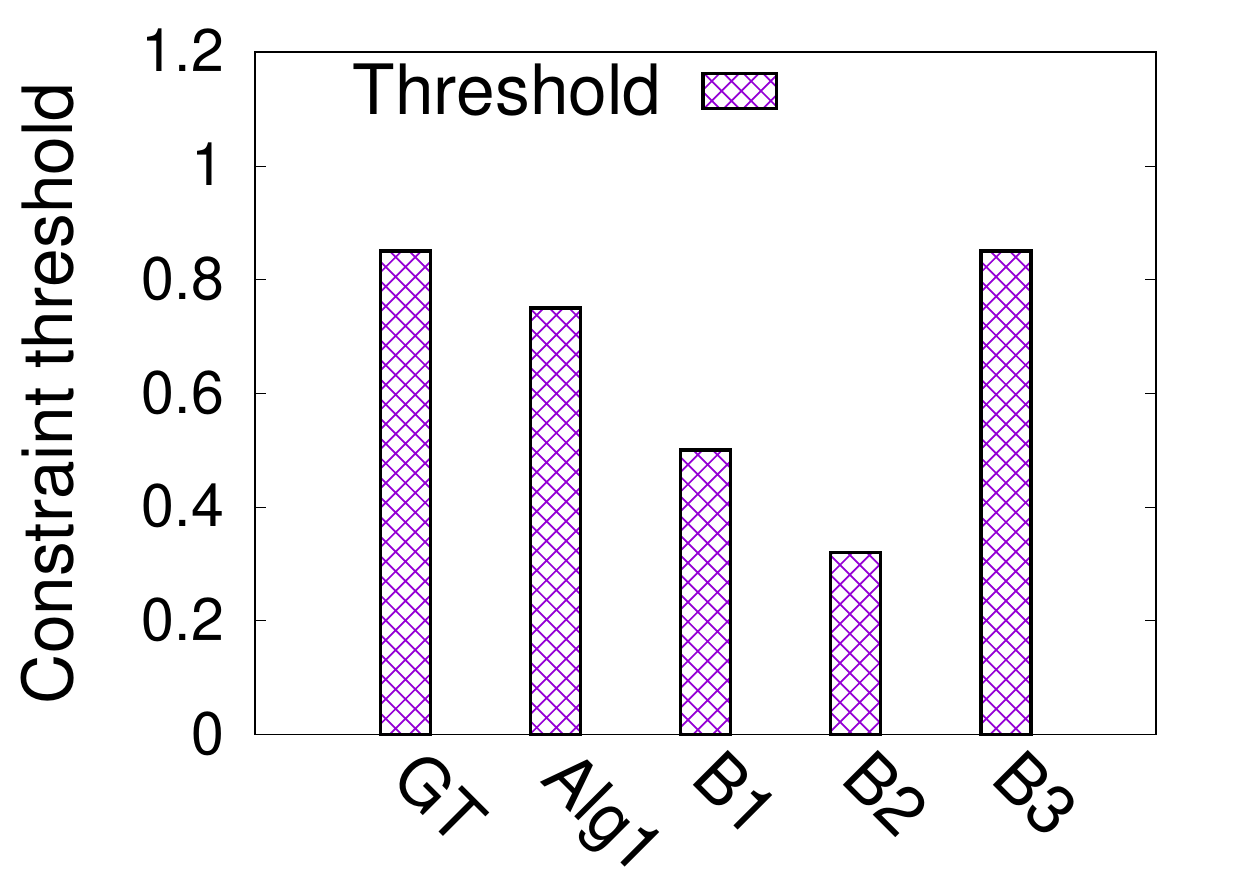}\label{constraint-adult2}}
\vspace{-4pt}
\caption{Comparison of estimated constraints for different datasets.}
\label{fig:likelihood}
\end{figure*}

\section{Experiment Setup} \label{sec:exp}
In this section, we evaluate the effectiveness of \sys{} on three real world datasets. We show that our techniques efficiently calculate the true likelihood of each constraint and the generated set of clusters are closer to the desired output, compared to other baselines.

\paragraph{\textbf{Datasets}} We evaluate our approach on three datasets, which are borrowed from the prior work that experiment with the metrics of interest.

\begin{itemize}
\item \textbf{Bank} dataset~\cite{bera2019fair} containing 4521 data nodes corresponding to phone calls from a marketing campaign by a Portuguese banking institution. The marital status of the records is considered as the sensitive feature for $\omega_{GF}$ constraint, with parameters $ [0.49,0.51]$.
\item \textbf{Adult} dataset~\cite{saisubramanian2020balancing} containing  $1000$ records with the income information of individuals along with their demographic attributes. `Age', `occupation', and `income' features are considered as the features of interest. Fairness constraint $\omega_{EQ}$ is optimized with respect to `occupation' and $\omega_{IC}$ with respect to `age' and 'income'.

\item \textbf{Crime} dataset~\cite{saisubramanian2020balancing} contains crime information about different $1994$ communities in the United States, where `number of crimes per 100K population' is used for $\omega_{IC}$ fairness constraint.
\end{itemize}
The features in these datasets are considered to calculate distance between every pair of nodes. Euclidean distance is calculated between numerical attributes and Jaccard distance between the categorical attributes. Please refer to \cite{bera2019fair,saisubramanian2020balancing} for more details.

\paragraph{\textbf{Baselines}}
We compare the results of our techniques with the following baselines: 
\begin{itemize}
	\item \texttt{B1} calculates the likelihood by considering each demonstration as a separate set of clusters;
	\item \texttt{B2} merges the different clusters in the demonstration to identify $k$ clusters and infers the likelihood over the identified clusters; and
	\item \texttt{B3} performs a grid search over all possible fairness constraints and identifies the clustering that conforms with the generated demonstrations.
\end{itemize}
Algorithm~\ref{alg:likelihood} is referred as \texttt{Alg1} and Algorithm~\ref{alg:greedy2} is labeled \texttt{Alg2} in all the plots in this section. Unconstrained k-center clustering technique is labeled as \texttt{kC}.

\paragraph{\textbf{Setup}} We use open source implementations of $\omega_{GF}$ and $\omega_{IC}$, and contacted the authors of \cite{galhotra2019lexicographically} for $\omega_{EQ}$. Their code base were used to generate ground truth clusters for an input constraint requirement. All algorithms were implemented in Python and tested on an Intel Core i5 computer with 16GB of RAM.  

Our experiments compare the identified fairness parameter by our algorithm and each baseline. To compare the quality of identified clusters, we compute the F-score of the identified intra-cluster pairs of nodes. F-score denotes the harmonic mean of the precision and recall, where precision refers to the fraction of correctly identified intra-cluster pairs and recall refers to the fraction of intra-cluster pairs that are identified by our algorithm. In all experiments, we report results with $k=5$. We execute the code of constraint clustering techniques with specified parameters to generate ground truth clustering. Each demonstration is generated by sampling a subset of \emph{five} nodes randomly from these clusters. Unless otherwise specified, we consider $2\log n$ demonstrations as input and these demonstrations do not reveal the true cluster affiliation of the considered nodes. In case there are multiple constraints that generate the same set of demonstrations, the algorithm output is considered correct if it correctly identifies any one of those constraints\footnote{Among the considered constraints, this situation does not arise whenever $|\Lambda|>5$}.

\section{Results and Discussion}

\paragraph{Effectiveness of Algorithm~\ref{alg:likelihood}}
The effectiveness of Algorithm~\ref{alg:likelihood} is measured based on the constraint threshold and the quality of the generated clusters.  Figure~\ref{fig:likelihood} compares the estimated threshold of the most-likely constraint, calculated by Algorithm~\ref{alg:likelihood} with  the ground truth and other baselines. Across all datasets, Algorithm~\ref{alg:likelihood} estimates the optimal threshold for every considered constraint, matching the performance of ground truth. This validates the effectiveness of Algorithm~\ref{alg:likelihood} to correctly estimate the most likely constraint and its corresponding threshold.

Among the baselines, \texttt{B3} achieves a similar performance. This is an expected behavior since \texttt{B3} performs naive grid search to explore all threshold values. Although it is effective in inferring the threshold, this technique is orders of magnitude inefficient due to the exhaustive enumeration of clusters using the different sets of constraints, features and their respective thresholds. It is therefore practically infeasible to implement this for problems with large input graphs and large $\Omega$. 

The other baselines \texttt{B1} and \texttt{B2} consistently show poor performance. Baseline \texttt{B1} does not identify any fairness constraint in settings where the demonstrations obey $\omega_{GF}$ and $\omega_{EQ}$ (Figure~\ref{constraint-bank} and \ref{constraint-tf} respectively). However, it identifies the optimal clustering constraint only in case of $\omega_{IC}$. 
Given that each demonstration has fewer than $5$ nodes, the information available in a single demonstration is not sufficient for \texttt{B1} to infer the true fairness constraint. On the other hand, \texttt{B2} overcomes the limitations of \texttt{B1} by merging the demonstrations randomly in order to capture constraint information over all demonstrations collectively. This approach has better performance than \texttt{B1} but does not identify the true clustering constraint in majority of the cases. It does not identify the fairness constraint $\omega_{EQ}$ (Figure~\ref{constraint-tf}) and the identified constraint threshold in all other cases are sub-optimal.

Figure~\ref{fig:accuracy} compares the quality of the returned clusters, by comparing the F-score of the clustering output of each technique with the ground truth clusters. In this experiment, Algorithm~\ref{alg:likelihood}  and \texttt{B3} achieve optimal performance as they identify the true ground truth clusters across all parameter settings. All other baselines did not identify the clusters correctly and achieved low F-score. Particularly, in case of $\omega_{EQ}$ and $\omega_{GF}$, the baselines \texttt{B1} and \texttt{B2} did not identify the optimal constraint threshold and generated  biased clusters.

\begin{figure}
\subfigure[\texttt{Adult} $\omega_{IC}$, $\beta=1$]{\includegraphics[width=0.23\textwidth]{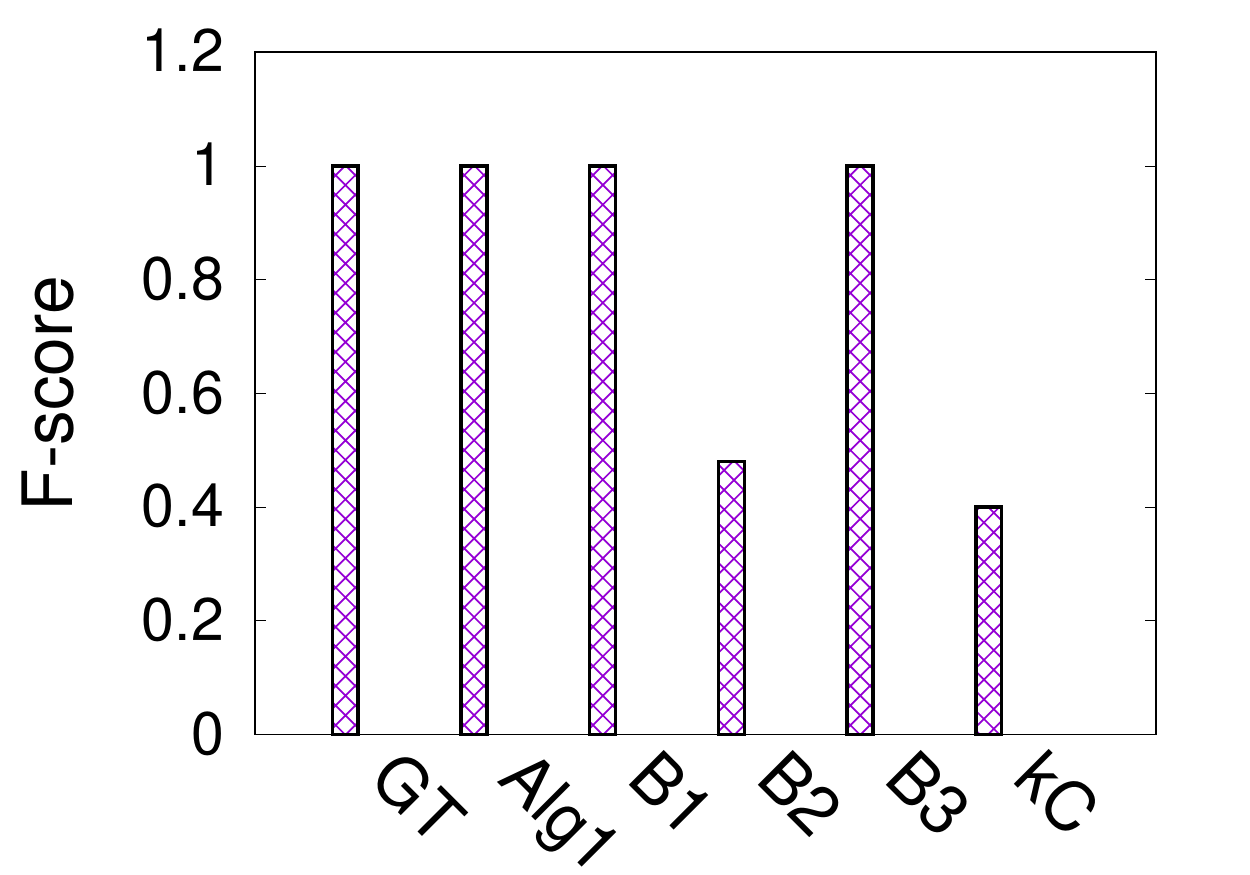}}
\subfigure[\texttt{Adult} $\omega_{IC}$, $\beta=0.9$]{\includegraphics[width=0.23\textwidth]{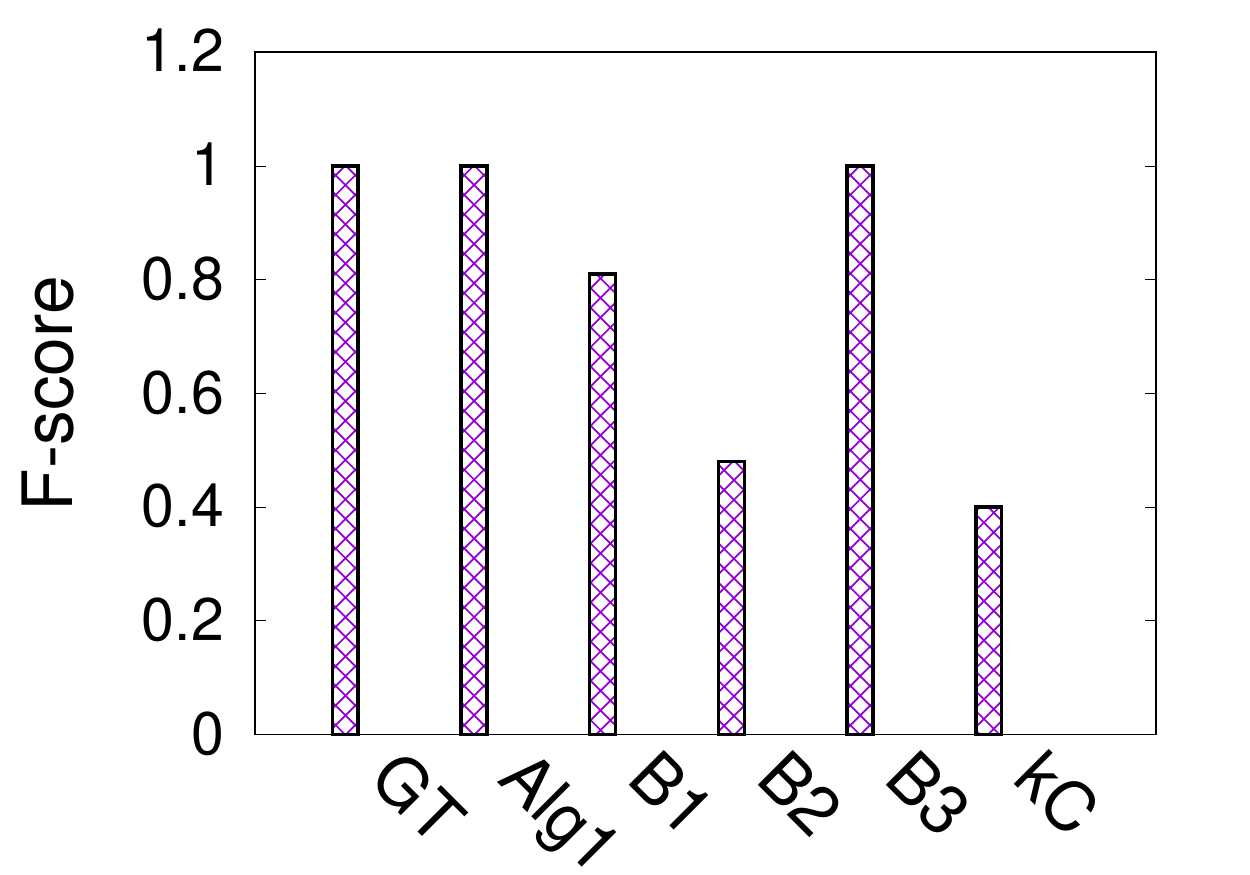}}
\vspace{-4pt}
\caption{F-score comparison for different datasets.}
\label{fig:accuracy}
\vspace{-4pt}
\end{figure}

\begin{table}
	\begin{center}
		\begin{tabular}{ |c|c|c|c|c| } 
			\hline
			Dataset & \texttt{Alg1} & \texttt{B1}& \texttt{B2}& \texttt{B3} \\ 
			\hline\hline
			\texttt{Bank} & 0.57  & 0.49 & 0.52 & 100 \\ \hline
			\texttt{Adult} & 1.14 & 1.01& 1.1 & 117\\
			\hline
			\texttt{Crime}  &1.02& 0.9 & 0.97& 104\\ 
			\hline
		\end{tabular}
	\end{center}
	\vspace{-4pt}
	\caption{Running time results (in minutes). \label{tab:efficiency}}
	\vspace{-4pt}
\end{table}

Table~\ref{tab:efficiency} compares the running time of \texttt{Alg1} and other baselines for different datasets and clustering constraints.  Among all datasets, \texttt{Alg1} is orders of magnitude faster than \texttt{B3}. In the worst case, \texttt{Alg1} generates $O(|\Omega|\times |F|)$ sets of clusters whereas \texttt{B3} generates clusters exhaustively for every value of constraint threshold.
The running time of \texttt{Alg1} is comparable with \texttt{B1} and \texttt{B2}.

\paragraph{Effectiveness of Algorithm~\ref{alg:greedy2}} \texttt{Alg2} identifies $k$ clusters such that the returned output obeys the fairness constraint reflected from the demonstrations $\Lambda$.   Figure~\ref{fig:demonstration-alg2} plots the F-score of \texttt{Alg2} for two data sets and the results are compared with that of \texttt{Alg1}. This allows us to compare the performance of our greedy \texttt{Alg2} with that of an existing efficient solver. In Figure~\ref{alg2-bank}, we employed the approach used in \citet{bera2019fair} to generate the ground truth clusters according to $\omega_{GF}$ and tested the effectiveness of \texttt{Alg2} to recover ground truth clusters for varying number of demonstrations. Similarly in Figure~\ref{alg2-adult}, ground truth is generated using $\omega_{IC}$.

When the number of demonstrations is less than $5$, the F-score of the generated clusters is $0.55$ for both domains. As we increase the number of demonstrations, we observe that the performance of \texttt{Alg2} improves and is closer to that of \texttt{Alg1}. \texttt{Alg2} achieves more than $0.9$ F-score in less than $20$ demonstrations. The continuous improvement in accuracy demonstrates the effectiveness of \texttt{Alg2} in recovering clusters without relying on a clustering algorithm. 
\begin{figure}
\subfigure[\texttt{Bank}, $\omega_{GF}$]{\includegraphics[width=0.23\textwidth]{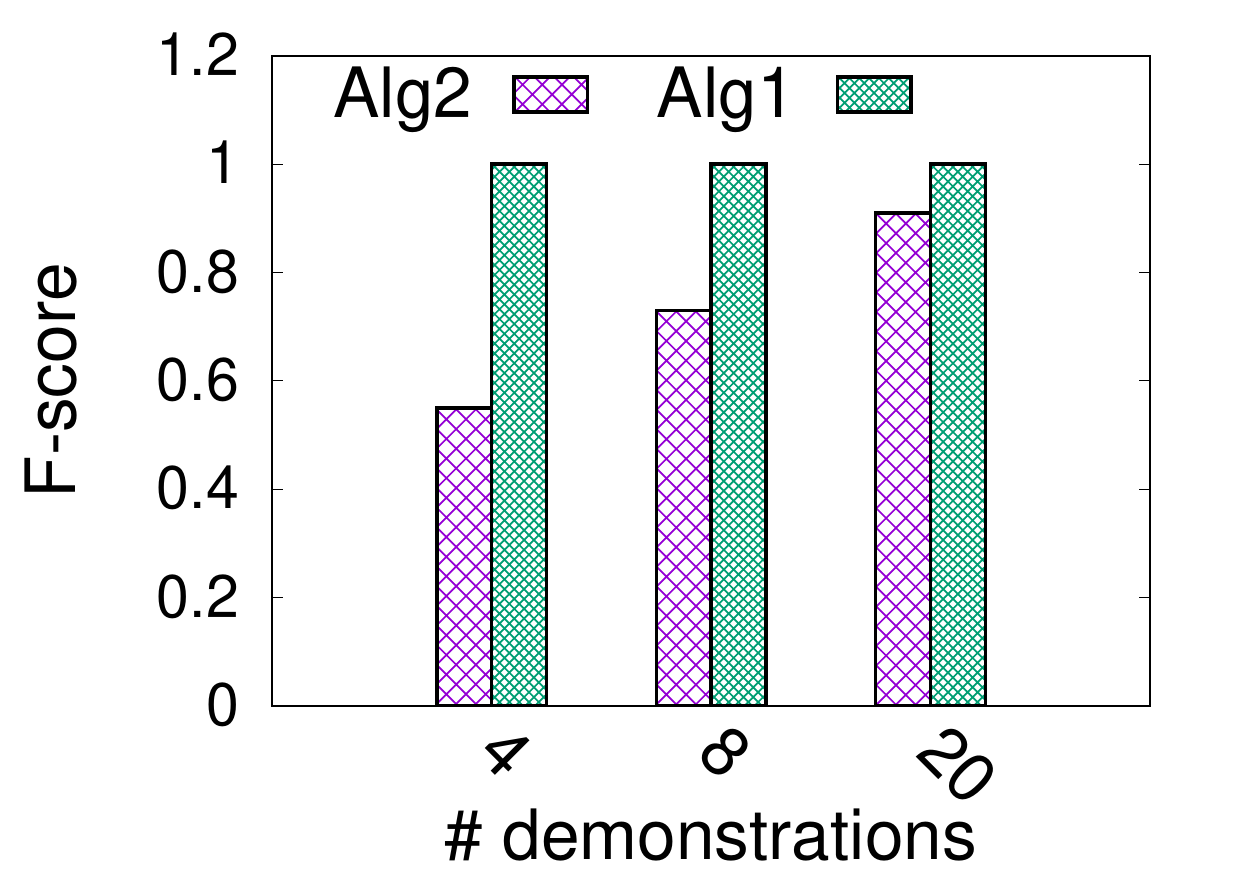}\label{alg2-bank}}
\subfigure[\texttt{Adult},  $\omega_{IC}$]{\includegraphics[width=0.23\textwidth]{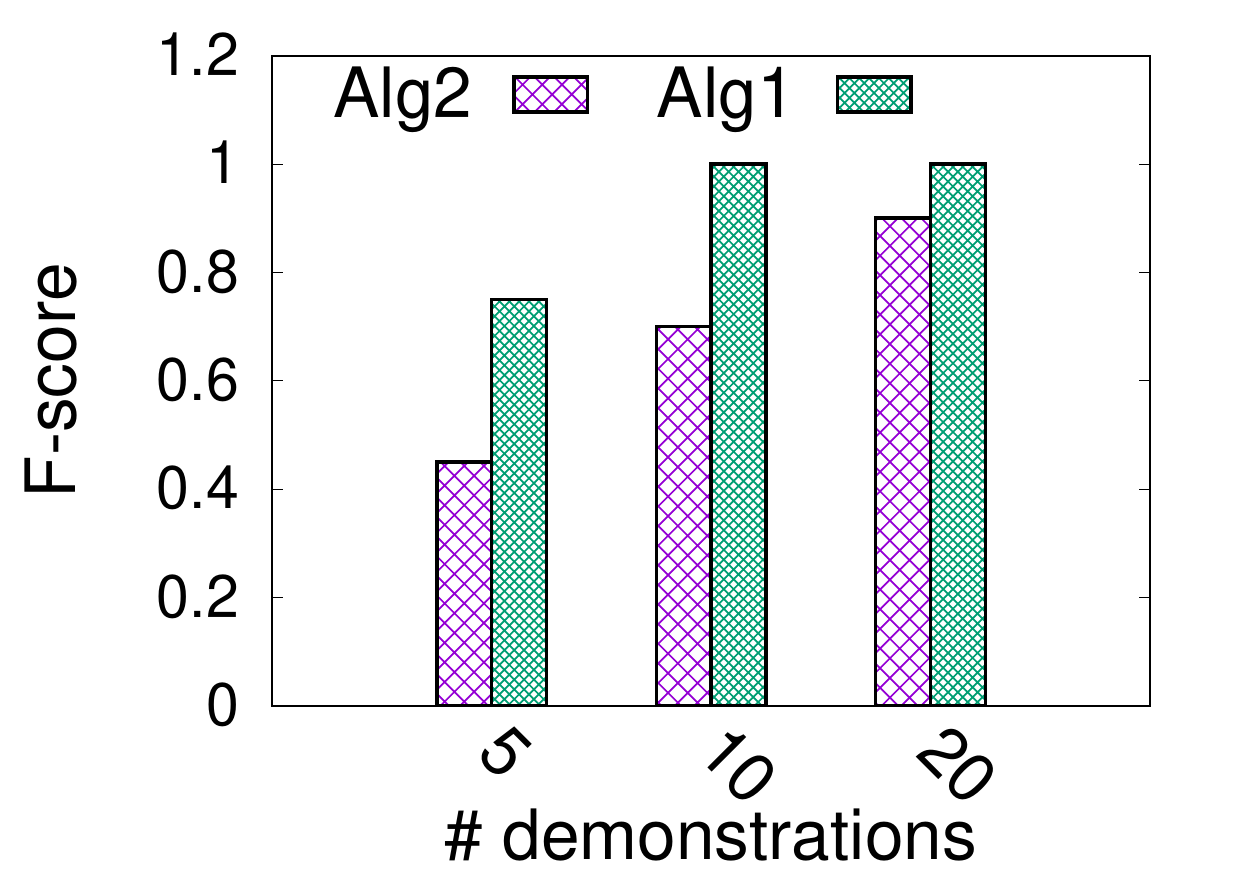}\label{alg2-adult}}
\vspace{-4pt}
\caption{Effect of \# demonstrations on \texttt{Alg2} performance.}
\vspace{-4pt}
\label{fig:demonstration-alg2}
\end{figure}

\paragraph{Effect of Demonstrations} 
We now investigate the effect of number of demonstrations on the performance of our techniques in identifying the optimal constraint threshold.  We varied the number of demonstrations in multiples of $\log n$: $0.5\log n, \log n, 2\log n$.
Figure~\ref{fig:demonstration} compares the constraint threshold and the F-score of the identified clusters using \texttt{Alg1}, with varying number of demonstrations on the \texttt{Bank} and \texttt{Adult} dataset. In case of $\omega_{GF}$, the ground truth constraint requires equal representation of the different groups in each cluster. Algorithm~\ref{alg:likelihood} correctly identifies the fairness constraint and achieves perfect F-score even when $\Lambda$ contains as low as four demonstrations. Increasing the number of demonstrations does not improve its performance as the constraint likelihood has already converged. In $\omega_{IC}$, the ground truth clusters are generated according to threshold $\beta = 0.85$. When the number of input demonstrations $|\Lambda|$ is low ($|\Lambda|=4$), the estimated interpretability constraint threshold is inaccurate and the constraint estimation improves as the number of demonstrations are increased. Algorithm~\ref{alg:likelihood} is able to achieve an F-score more than $0.8$ with just ten demonstrations and the quality of final clusters improves monotonically with increasing demonstrations. It converges to the accurate constraint threshold whenever $|\Lambda|\geq 20$ and therefore achieves perfect F-score. 

In Figure~\ref{fig:likelihood}, the input demonstrations do not reveal the true cluster affiliation of any of the nodes. We ran an additional experiment with the globally informative demonstrations (Definition~\ref{defn:global}), which reveals the ground truth cluster affiliation of each node in the demonstration. With this additional information, we observe that Algorithm~\ref{alg:likelihood} converges to the optimal constraint threshold in less than ten demonstrations. This experiment validates that Algorithm~\ref{alg:likelihood} is able to leverage the extra information provided by globally informative demonstration to converge faster. 

\begin{figure}
\subfigure[\texttt{Bank}, $\omega_{GF}$]{\includegraphics[width=0.23\textwidth]{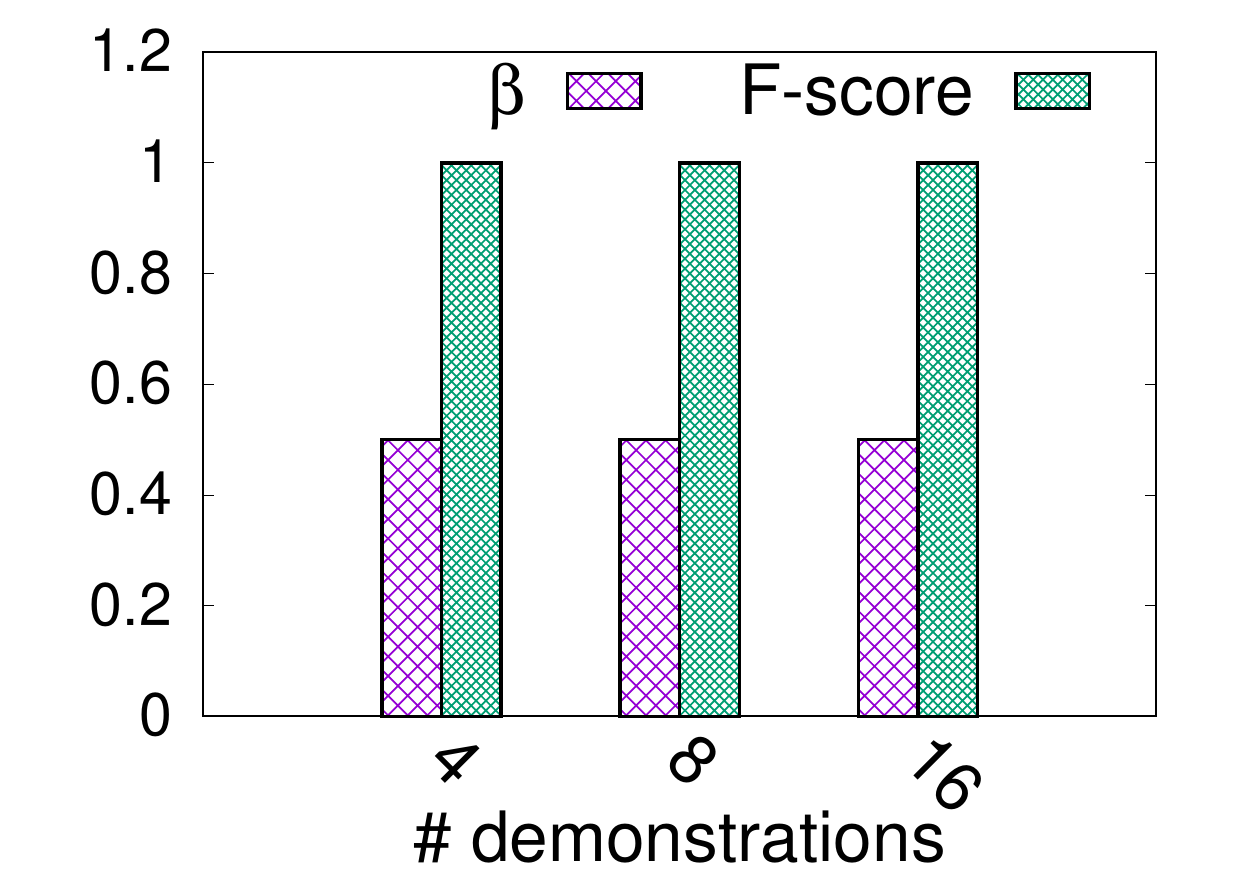}}
\subfigure[\texttt{Adult}, $\omega_{IC}$ ]{\includegraphics[width=0.23\textwidth]{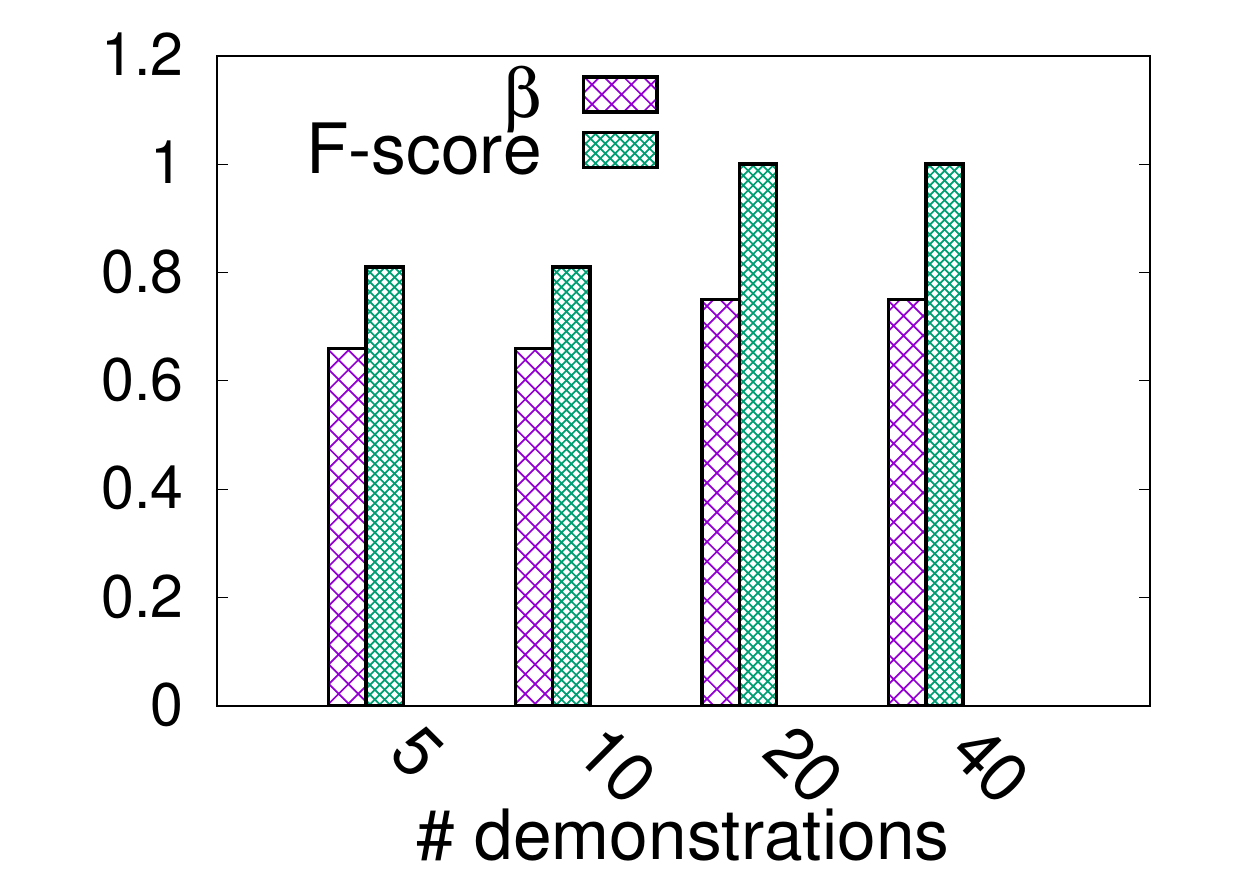}}
\vspace{-4pt}
\caption{Effect of \# demonstrations on \texttt{Alg1} performance.}
\vspace{-4pt}
\label{fig:demonstration}
\end{figure}

\begin{figure}
\subfigure[\texttt{Bank}, $\omega_{GF}$]{\includegraphics[width=0.23\textwidth]{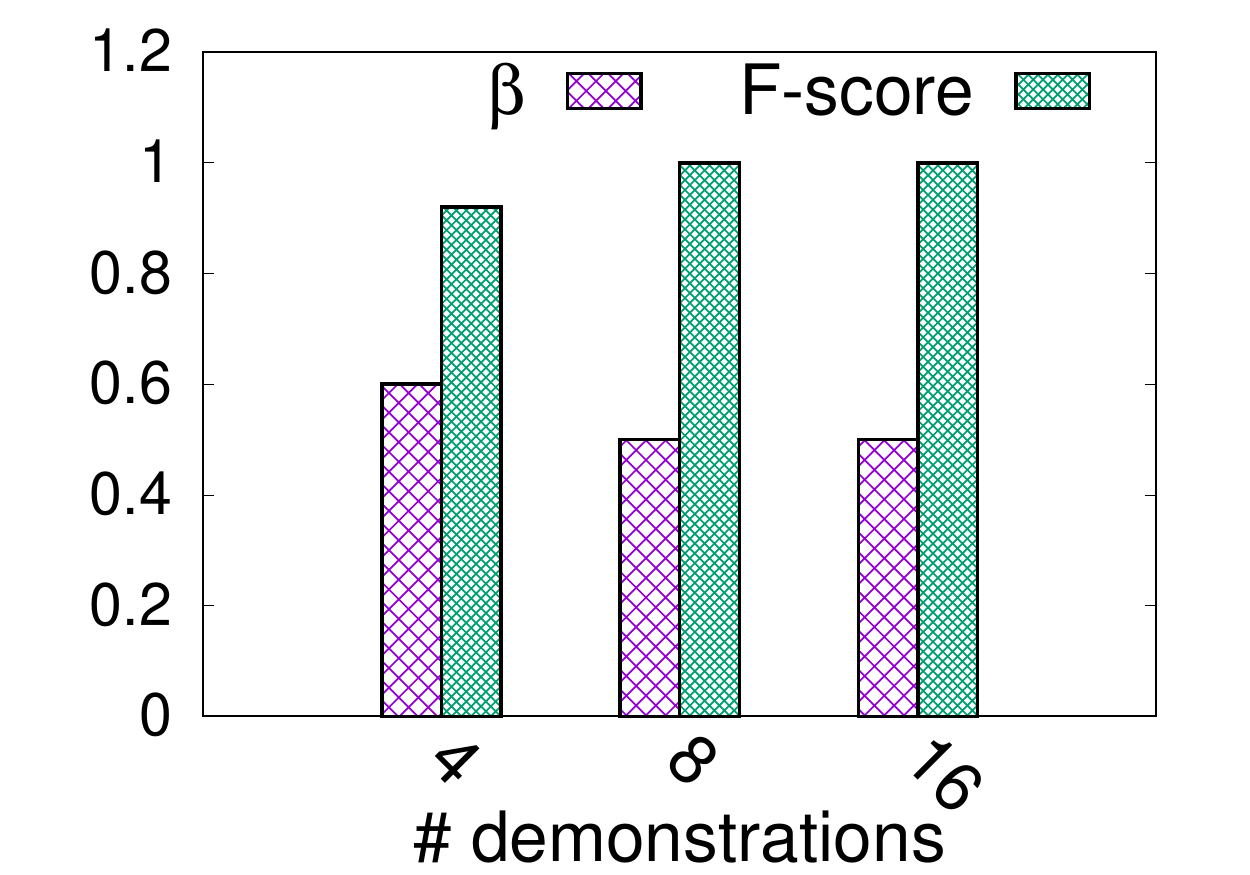}}
\subfigure[\texttt{Adult}, $\omega_{IC}$ ]{\includegraphics[width=0.23\textwidth]{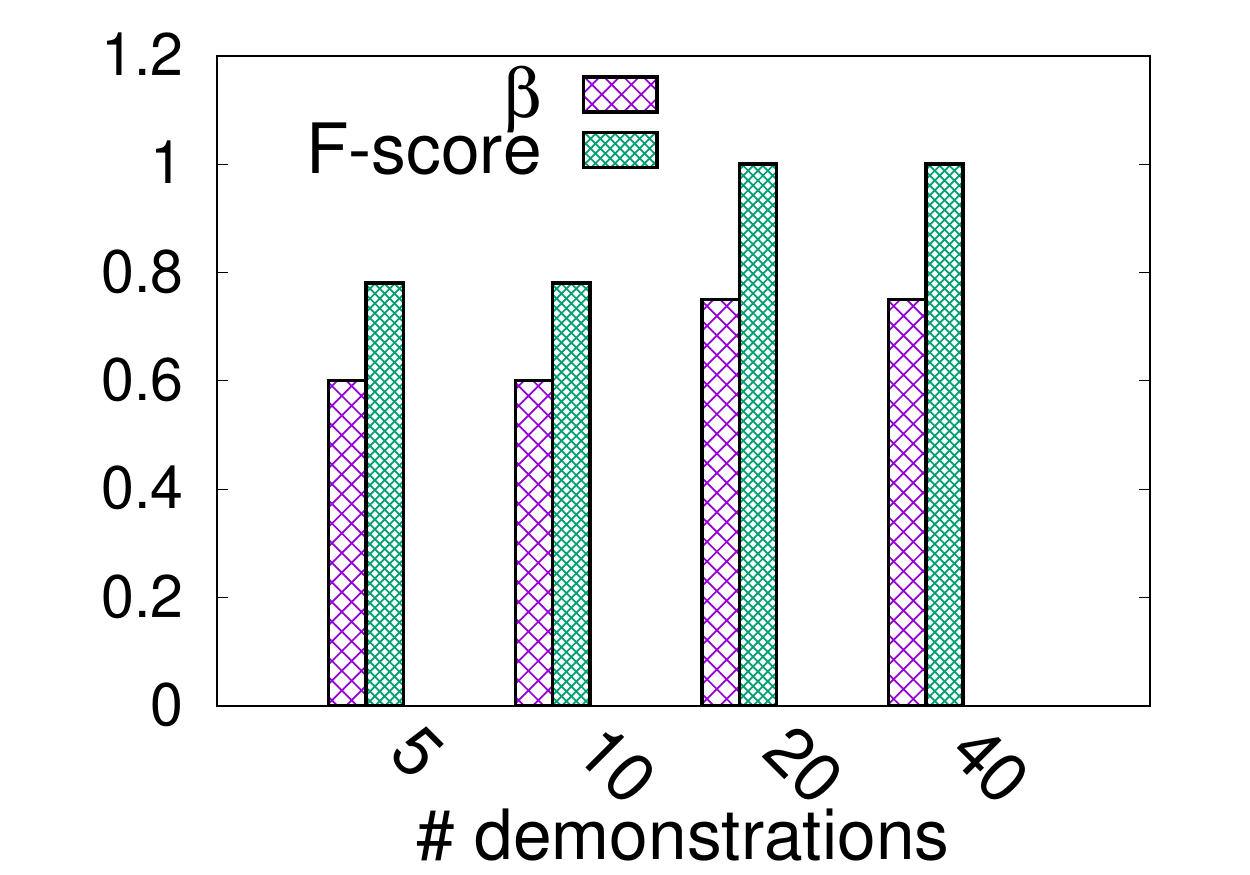}}
\vspace{-4pt}
\caption{Effect of sampling bias on \texttt{Alg1} performance.}
\vspace{-4pt}
\label{fig:biased}
\end{figure}

Next, we evaluate  the effect of number of demonstrations on the performance of Algorithm~\ref{alg:greedy2}. Figure~\ref{fig:demonstration-alg2} shows that as we increase the number of demonstrations, \texttt{Alg2} matches the F-score of \texttt{Alg1}.

\paragraph{Ablation Study} To test the effectiveness of our constraint estimation techniques, we varied the size of demonstration from $4$ to $10$ for the different constraints. As expected, the number of required demonstrations reduces linearly with increase in demo size\footnote{We do not consider smaller demonstrations because clustering fewer than $4$ nodes do not reveal information about the underlying clusters.}.  Therefore, an increase in demonstration size helps \texttt{Alg1} converge faster.

We tested the robustness of our constraint threshold estimation techniques by generating demonstrations according to a biased distribution. In the first experiment (Figure~\ref{fig:biased}), we employed a \emph{biased} sampling procedure, where each demonstration is biased in favor of some specific clusters but all nodes within those chosen clusters are equally likely to be chosen for the demonstration. Specifically, we follow a two step procedure where we first sample the cluster $C_i$ with probability $p_i $ and the nodes from the sampled cluster are chosen randomly. This introduction of bias did not affect the quality of our techniques and \texttt{Alg1} was able to recover ground truth clusters in $\Theta(\log n)$ demonstrations. The second experiment considered a biased sampling procedure where the expert samples fewer nodes from the marginalized groups. For example, a node having `red' color is chosen with probability $\frac{1}{n}$  but a blue colored node is chosen with probability $\frac{4}{n}$. In such setting, the returned demonstrations are biased against the marginalized groups and the inferred clustering threshold is not accurate. We observe that this bias translates into the constraint threshold estimation procedure of \texttt{Alg1}. This experiment justifies the requirement of an unbiased expert annotator that chooses nodes randomly, without considering their sensitive attributes.

To further study the effect of $k$, we vary the number of clusters as $k=\{5,10,15,20,50\}$ for \texttt{adult} dataset and calculated the number of demonstrations required to identify the true clustering constraint. For all values of $k$, \texttt{Alg1} identified the optimal set of clusters in less $20$ ($2\log n$) demonstrations and the number of required demonstrations increases sub-linearly with $k$.  For example, it required $20$ demonstrations for $k=5$ and $60$ demonstrations were enough for $k=50$. This increase in number of demonstrations is justified because \texttt{Alg1} tries to merge presented demonstrations into $k$ clusters. If the number of clusters in presented demonstrations is smaller than $k$, then it might end up partitioning some clusters which may introduce some noise in the likelihood estimation procedure. However, when the input demonstrations are globally informative, the number of required demonstrations do not increase with $k$ and therefore do not include the plots. \texttt{Alg1} converges to the optimal clustering constraint as soon as there are $\Theta(\log n)$ nodes from any of the clusters. 

\begin{figure}
	\centering
{\includegraphics[scale=0.5]{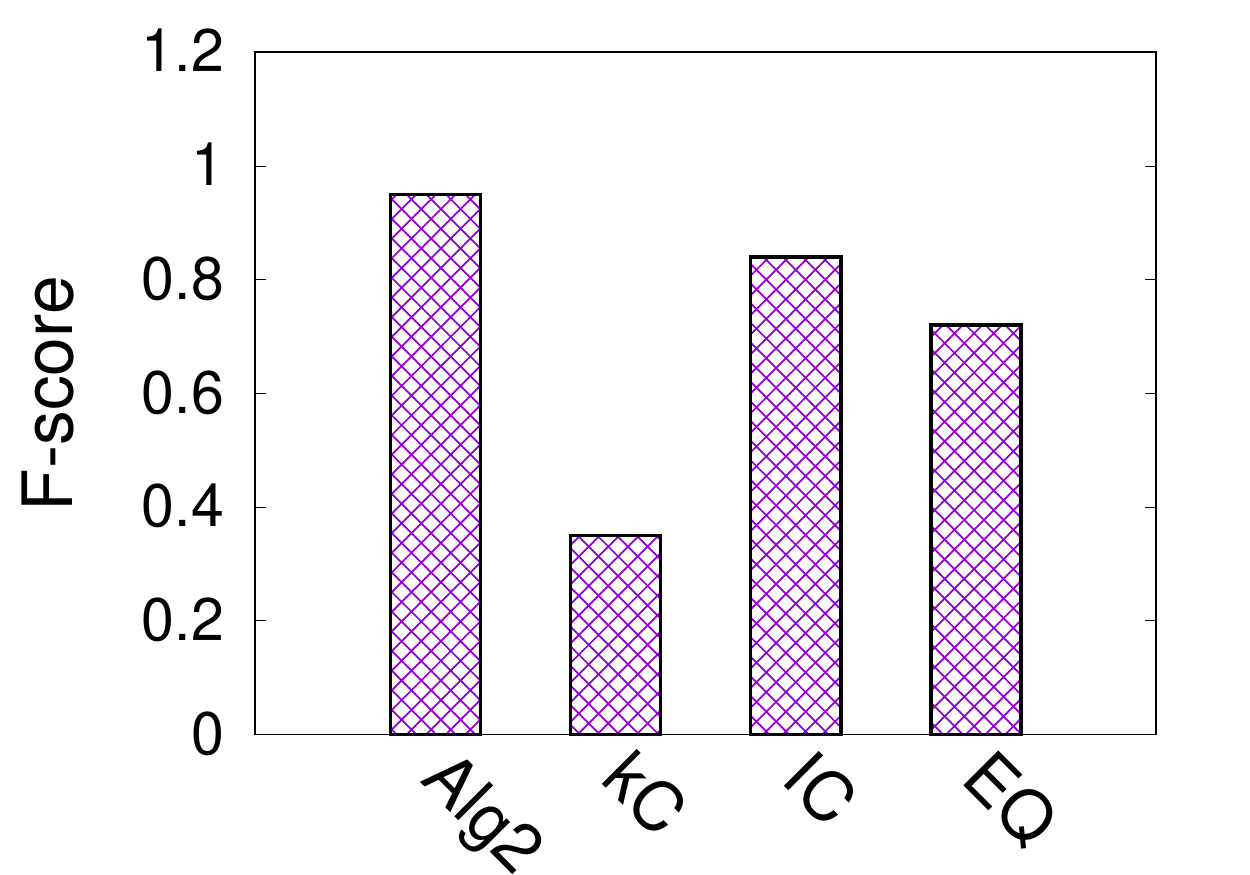}}
\caption{F-score of fair and interpretable clusters generated by different techniques.}
\label{fig:iceq}
\vspace{-4pt}
\end{figure}

\paragraph{Fair and Interpretable Clusters} To further evaluate the effectiveness of generating fair and interpretable clusters, we ran interpretable clustering algorithm~\cite{saisubramanian2020balancing} with $\beta\!=\!1$ for \texttt{Adult} dataset. The generated clusters were then post-processed to satisfy $\omega_{EQ}$. Since none of the current clustering algorithms optimize for fairness and interpretability, we implemented a greedy technique to process the output of interpretable clusters and satisfy fairness constraint. We considered this output as the ground truth to generate globally informative demonstration $\Lambda$ and ran \texttt{Alg2} to calculate the set of clusters with maximum likelihood. \texttt{Alg2} achieved F-score of more than $0.9$ (Figure~\ref{fig:iceq}) with less than $25$ demonstrations, each with $5$ nodes. Any baseline that optimizes $\omega_{IC}$ or $\omega_{EQ}$ alone achieve sub-optimal performance. This experiment demonstrated the ability of \texttt{Alg2} to generate clusters even when the constraint optimization algorithm is not known. Additionally, \texttt{Alg2} requires the expert to label less than $25\%$ dataset to generate fair and interpretable clusters.

\section{Summary and Future Work}
With the availability of many nuanced fairness definitions, it is non-trivial to specify a fairness metric that captures what we intend. As a result, systems may be deployed with an incomplete specification of the fairness metric, which leads to biased outcomes. We formalize the problem of inferring the complete specification of the fairness metric that the designer intends to optimize for a given problem. We present an algorithm to generate fair clusters by inferring the fairness constraint using expert demonstration and analyze its theoretical guarantees. We also present a greedy approach to generate fair clusters for objectives which are not currently supported by the existing suite of fair clustering algorithms. To the best of our knowledge, our algorithm is the first to combine graph clustering and learning from demonstrations, particularly to improve fairness. We empirically demonstrate the effectiveness of our approach in inferring fairness and interpretability metrics, and then generate clusters that are fair and interpretable. Although we discuss the framework in the context of fair clustering, our proposed framework can be used to infer any clustering constraints, as shown in the experiments. 

In the future, we plan to conduct a human subjects study to evaluate our approach and design robust algorithms to infer the intended metrics in the presence of noise. Developing robust techniques to handle bias in demonstrations is another interesting question for future work. Extending our algorithm to handle other fairness metrics and interpretability metrics will broaden the scope of problems that can be handled by our approach.

\clearpage
\bibliographystyle{aaai21}
\bibliography{LCDReferences}
\end{document}